\relax
\documentclass[accepted]{uai2021}

\usepackage[american]{babel}

\usepackage{natbib} 
\usepackage{mathtools} 
\usepackage{booktabs} 
\usepackage{tikz} 

\usepackage{times}  
\usepackage{helvet} 
\usepackage{courier}  
\usepackage{natbib}  
\usepackage{caption} 

\usepackage{appendix}
\usepackage{thm-restate}

\usepackage{graphicx}
\usepackage{bbm}
\usepackage[utf8]{inputenc}
\usepackage{amsthm}
\usepackage{amsmath}
\usepackage{amsfonts}
\usepackage{booktabs}
\usepackage{xcolor}
\usepackage{subcaption}
\usepackage{delimset}
\usepackage{mathtools}
\usepackage[capitalise]{cleveref}
\usepackage{mathtools, amssymb, nccmath}
\usepackage{xparse}
\usepackage[switch]{lineno}  %

\usepackage{algorithm}
\usepackage{algcompatible}
\usepackage{setspace}
\let\Algorithm\algorithm
\renewcommand\algorithm[1][]{\Algorithm[#1]\setstretch{1.5}}

\newcommand{\A}{\mathcal{A}}
\newcommand{\X}{\mathcal{X}}
\newcommand{\F}{\mathcal{F}}
\newcommand{\D}{\mathcal{D}}

\newcommand{\xBound}{S_x}

\newcommand{\xoBound}{S_{x,o}}
\newcommand{\wBound}{S_w}

\newcommand{\woBound}{S_{w,o}}

\newcommand{\R}{\mathbb{R}}
\newcommand{\rank}{{\mathrm{rank}}}

\newcommand{\Olog}{\widetilde{\mathcal{O}}}

\renewcommand{\det}[1]{\text{det}\pth{#1}}
\newcommand{\pth}[1]{\left( #1 \right) }
\newcommand{\indicator}[1]{\mathbbm{1}_{\brk[c]*{#1}}}
\renewcommand{\log}[1]{\mathrm{log}\pth{#1}}
\renewcommand{\lg}[1]{\mathrm{log}_2\pth{#1}}

\newcommand{\inner}[1]{\brk[a]*{#1} }
\newcommand{\pert}[1]{\hat {#1}}

\newcommand{\Regret}[1]{\textit{Regret}\pth{#1}}
\newcommand{\RegretN}[1]{\textit{Regret}_{#1}}

\newcommand\trace[1]{\text{trace}\pth{#1}}
\renewcommand\norm[1]{\left\lVert#1\right\rVert}

\newcommand{\parallelsum}{\mathbin{\|}}

\newcommand{\xk}{x^\mathrm{o}}
\newcommand{\xc}{x^\mathrm{h}}

\newcommand{\HmuC}{\pert{\mu}_{\mathrm{h}\mid a}}
\newcommand{\Qk}{Q^\mathrm{o}}
\newcommand{\Qc}{Q^\mathrm{c}}

\newcommand{\G}{\mathcal{G}}

\DeclarePairedDelimiterXPP\expect[2]{\mathbb{E}^{#1}}[]{}{\setargs{#2}}%
\NewDocumentCommand{\setargs}{>{\SplitArgument{1}{|}}m}
{\setargsaux#1}
\NewDocumentCommand{\setargsaux}{mm}
{\IfNoValueTF{#2}{#1}{\nonscript\,#1\nonscript\;\delimsize\vert\nonscript\:\allowbreak #2\nonscript\,}}
\DeclarePairedDelimiterXPP\expectaux[3]{\mathbb{E}_{#1}}[]{}{#2\nonscript\:\delimsize\vert\nonscript\:#3}%

\newcommand{\E}[1]{ \expect*{}{#1} }
\newcommand{\Eb}[1]{ \expect*{\pi_b}{#1} }

\newtheorem{theorem}{Theorem}
\newtheorem*{theorem*}{Theorem}
\newtheorem{lemma}{Lemma}
\newtheorem{corollary}{Corollary}

\newtheorem{assumption}{Assumption}
\newtheorem{remark}{Remark}
\newtheorem*{remark*}{Remark}

\newtheorem{proposition}{Proposition}
\newtheorem{property}{Property}

\DeclarePairedDelimiter\br{(}{)}
\DeclarePairedDelimiter\brc{\{}{\}}

\makeatletter
\let\oldappendix\appendices
\renewcommand{\appendices}{%
  \clearpage
  \renewcommand{\thesection}{\Alph{section}}
  \let\tf@toc\tf@app
  \addtocontents{app}{\protect\setcounter{tocdepth}{2}}
  \immediate\write\@auxout{%
    \string\let\string\tf@toc\string\tf@app^^J
  }
  \oldappendix
}%
\newcommand{\listofappendices}{%
  \begingroup
  \renewcommand{\contentsname}{\appendixtocname}
  \let\@oldstarttoc\@starttoc
  \def\@starttoc##1{\@oldstarttoc{app}}
  \tableofcontents
  \endgroup
}

\makeatother

\title{Bandits with Partially Observable Confounded Data}

\author[1]{Guy Tennenholtz}
\author[1]{Uri Shalit}
\author[1,2]{Shie Mannor}
\author[1]{Yonathan Efroni}
\affil[1]{%
    Technion, Israel Institute of Technology
}
\affil[2]{%
    Nvidia Research
}

\begin{document}
\maketitle

\begin{abstract}
We study linear contextual bandits with access to a large, confounded, offline dataset that was sampled from some fixed policy. We show that this problem is closely related to a variant of the bandit problem with side information. We construct a linear bandit algorithm that takes advantage of the projected information, and prove regret bounds. Our results demonstrate the ability to take advantage of confounded offline data. Particularly, we prove regret bounds that improve current bounds by a factor related to the visible dimensionality of the contexts in the data. Our results indicate that confounded offline data can significantly improve online learning algorithms. Finally, we demonstrate various characteristics of our approach through synthetic simulations.
\end{abstract}

\section{Introduction}
\label{sec: intro}

The use of offline data for online control is of practical interest in fields such as autonomous driving, healthcare, dialogue systems, and recommender systems \citep{mirchevska2017reinforcement,murphy2001marginal,li2016learning,covington2016deep}. There, an abundant amount of data is readily available, potentially encompassing years of logged experience. This data can greatly reduce the need to interact with the real world, as such interactions may be both costly and unsafe \citep{amodei2016concrete}. Nevertheless, as offline data is usually generated in an uncontrolled manner, it poses major challenges, such as unobserved states and actions. Failing to take these into account may result in biased estimates that are confounded by spurious correlation \citep{gottesman2019guidelines}. This work focuses on utilizing partially observable offline data in an online bandit setting.

We consider the stochastic linear contextual bandit setting \citep{auer2002using,chu2011contextual,zhou2019learning}. Here, the context is a vector $x \in \R^d$ encompassing the full state of information. We assume to have additional access to an offline dataset in which only $L < d$ covariates (features) of the context are available. The unobserved covariates in the data are known as unobserved confounding factors in the causal inference literature \citep{bookofwhy}, which may cause spurious associations in the data, rendering the data useless unless further assumptions are made \citep{neuberg2003causality,shpitser2012counterfactuals,bareinboim2015bandits}. In this work we assume that, when interacting with the online environment, the full context is accessible, and search for methods to combine both sources of information (online and offline) to quickly converge to an optimal solution.

We construct an algorithm that is provably superior to an algorithm which does not utilize the (partially observable) information in the data. We recognize the following fundamental observation: \textbf{Confounded offline data can (still) be used to improve online learning}, and specifically, that partially observable offline data can be utilized as linear side information (linear constraints) for the bandit problem.

While the bandit setting with confounded offline data has already been explored, its combination with a fully observable online environment is a new setting with particular challenges and benefits. First, one cannot ensure identification of an optimal policy with confounded offline data (see Section~\ref{sec: from data to linear constraints}). This has implications on safety and applicability of algorithms which are based solely on offline data, e.g., the confounding bias of offline critical care datasets \citep{johnson2016mimic}. Second, in contemporary widespread applications, an abundant of offline data is readily available. These application do not necessarily prevent interactions with the real world. On the contrary, countless real-world applications can access the real world. Still, such interactions may be costly, time consuming, or unsafe. It is thus vital to utilize the enormous amounts of previously collected offline data to reduce as much as possible the need for online interactions. We discuss two concrete examples from the healthcare and traffic management domains below.

\begin{figure*}[t!]
\captionsetup[subfigure]{labelformat=empty}
\centering
\begin{subfigure}{.35\textwidth}
    \centering
    \includegraphics[width=\textwidth]{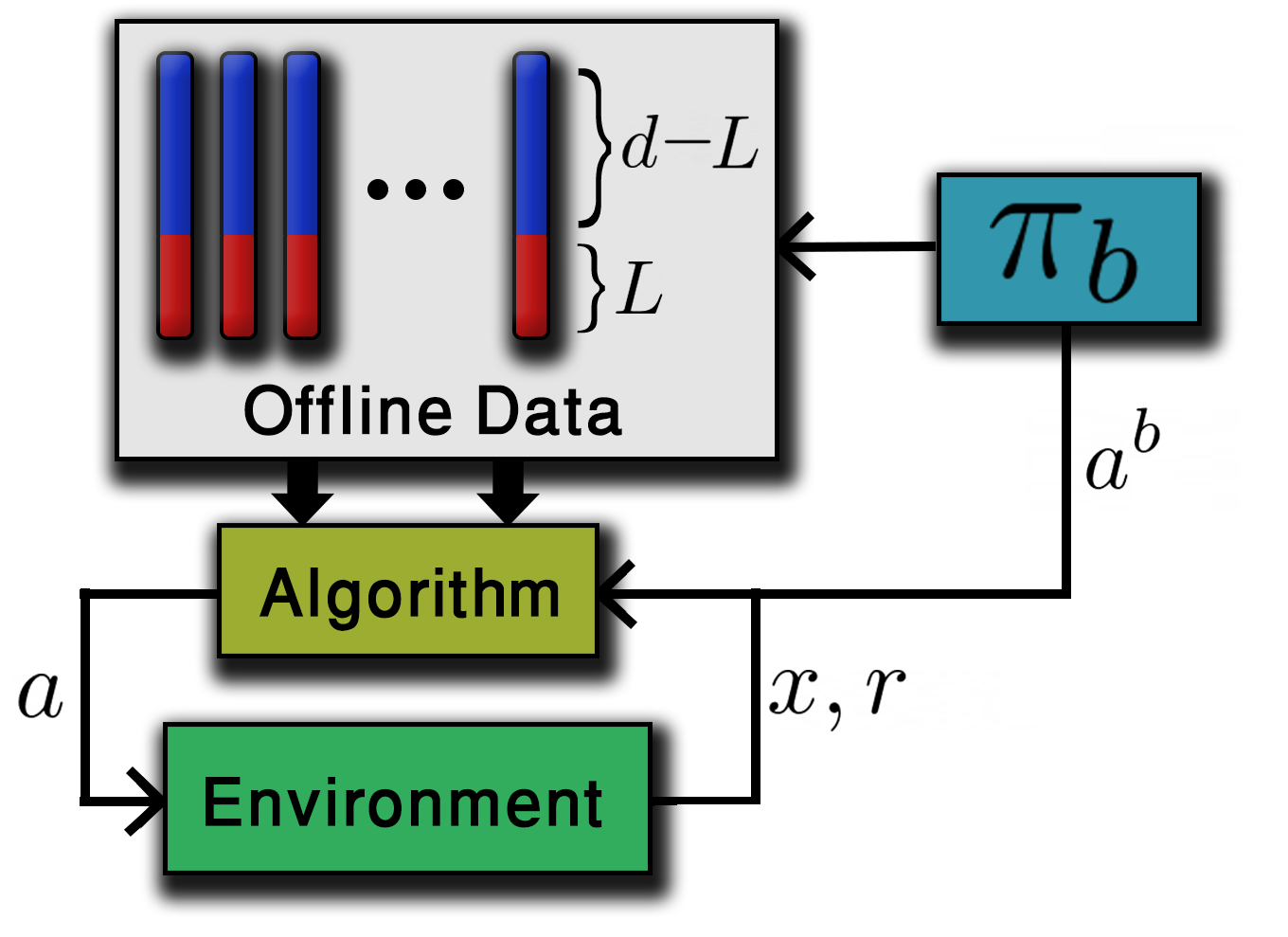}
    \caption{(a)}
\end{subfigure}
\begin{subfigure}{.35\textwidth}
    \centering
    \includegraphics[width=\textwidth]{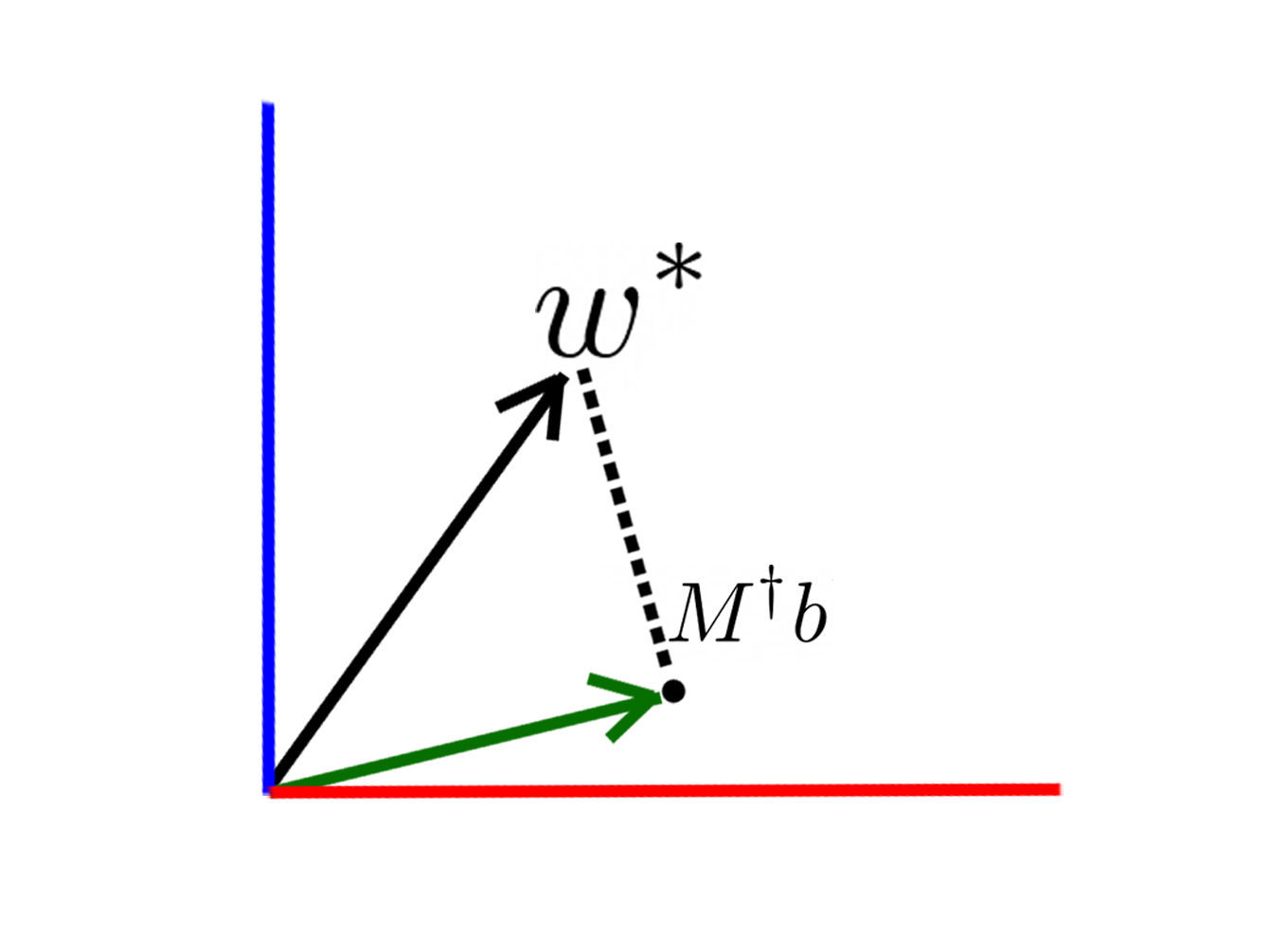}
    \caption{(b)}
\end{subfigure}
\centering
\caption{\label{fig: setup and projections} \footnotesize \textbf{(a)} Block diagram of our setup: an online learner interacting with an environment while utilizing partially observable offline data that was generated by a behavior policy $\pi_b$. \textbf{(b)} This plot depicts the projection of $Mw^* = b$. We show that partially observable offline data can provide us with approximate linear side information of this form. The online learner must then estimate the orthogonal subspace, attempting to reduce the effective dimensionality of the problem. }
\end{figure*} 

\noindent \textbf{Healthcare. } 
Consider the important challenge of cancer chemotherapy control; specifically, optimal drug dosing for cancer chemotherapy \citep{sbeity2015review}. Clinicians usually follow established guidelines for treating each patient, prescribing drug doses according to the stage of the tumor, the weight of the patient, white blood cell levels, concurrent illnesses, and the age of the patient. Suppose we are given access to large amounts of medical records of chemotherapy plans, specifying the frequency and dose of drug administration as well as their effect on the patient. Due to privacy regulations, the patients' socioeconomic characteristics are removed from the data. Nevertheless, these features may have affected the physician's decisions, as well as the outcome of the prescribed treatments.  Next, suppose we are able to interact with the world, where the full state of the patients' information is available to us. How would we efficiently construct an algorithm to automate chemotherapy treatment while also utilizing the partially observable, confounded data?

\noindent \textbf{Smart City Traffic Management. } 
Consider the problem of adjusting traffic signals based on real-time traffic conditions using video footage of cameras located over intersections. The development time of the system consists of continual addition of new labels (classes) for the different types of vehicles and pedestrians based on relevant characteristics that may affect traffic congestion. Due to this recurrent process, data that was gathered in previous times may render itself useless, outdated, and even harmless, unless handled properly. This is due to the fact that some of the new information in the state was not previously collected, yet is needed for training future control strategies. How should one use the partially observable historical data for improving the most recent online system?


In this work we show how the confounded information in the data can be utilized for the online bandit problem. Figures~\ref{fig: setup and projections} and~\ref{fig: block diagram} illustrate our basic setup and approach. We show how confounded offline data can be thought of as linear constraints to the online problem. These linear constraints, are not fully known. They are in fact dependent on the cross-correlation matrix of the context vector induced by policy that generated the data (which we denote as the behavior policy, $\pi_b$). To learn these constraints and utilize them, we approximate the cross-correlation matrix through online interactions and carefully integrate them into our learning algorithm, decreasing the overall regret.

The contributions of our work are as follows. As a fundamental contribution we propose a framework for combining confounded offline data with online learning. This framework is a gateway between fully confounded offline data to online learning, and encompasses a variety of important problems and applications. While this work only considers the linear bandit setting, it sets the building blocks and insights needed for more complex settings (e.g., reinforcement learning). Our second contribution shows that partially observable confounded data can in fact be realized as linear constraints for the online problem (see Section~\ref{sec: from data to linear constraints}). To the best of our knowledge, this work is the first to show this relation. Finally, we prove that the overall regret can indeed be decreased when using the confounded data. 
Our proof, too, consists of technical obstacles related to the approximate constraints, which must be learned simultaneously.

\begin{figure*}[t!]
\centering
\includegraphics[width=\textwidth]{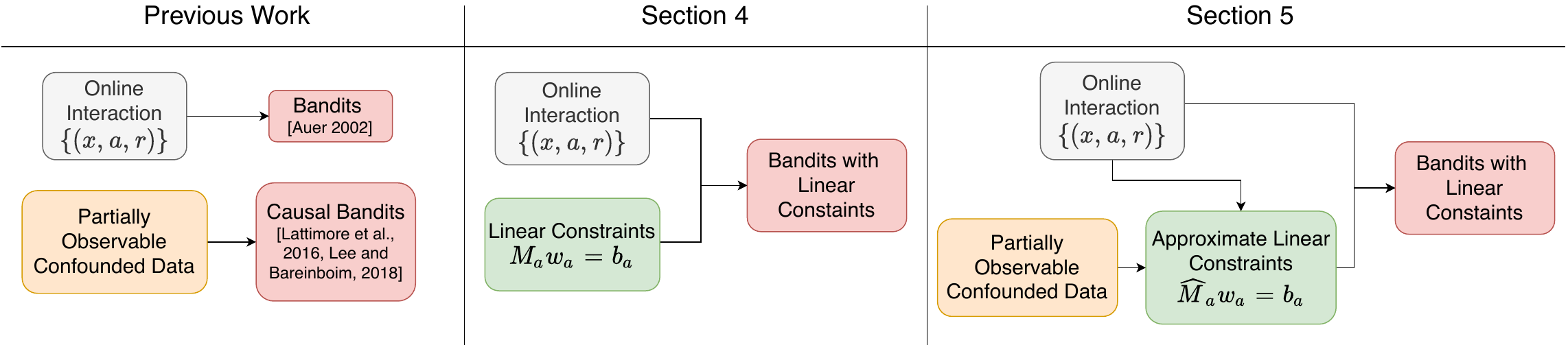}
\centering
\caption{ \label{fig: block diagram}  \footnotesize  Previous work has dealt with bandits in the online setting. Other work integrated causal information to utilize confounded offline data. This work combines the two through constraints on the online problem. In Section~\ref{sec: side information} we show how linear constraints can be leveraged to achieve better regret for the bandit problem (Theorem~\ref{thm: regret bound oful}). Then, in Section~\ref{section: partially observable covariates} partial linear constraints are estimated from online interactions, and then utilized efficiently by our learning algorithm. Note that $b_a$ is not estimated as it is previously computed from the offline data (see Section~\ref{sec: from data to linear constraints}). Finally, due to fast convergence of the linear constraints, improved performance is still achieved (Theorem~\ref{thm: oful with data}).}
\end{figure*}

\section{Problem Setting}
\label{sec: preliminaries}

\textbf{Notations. } 
We use $[n]$ to denote the set $\brk[c]*{1, \hdots, n}$. We denote by $I_m$ the $m \times m$ identity matrix. Let $y, z \in \R^d$ and $A, B \in \R^{d \times d}$. We use $\norm{z}_2$ to denote the $\ell_2$-norm and $z^T$ the transpose of $z$. The inner product is represented as $\inner{z, y}$. For $A$ semi-positive definite, the weighted $\ell_2$-norm is denoted by $\norm{z}_A = \sqrt{z^TAz}$. The minimum and maximum singular values of $A$ are denoted by $\lambda_{min}(A)$ and $\lambda_{max}(A)$ respectively. Furthermore, $A \preceq B$ if $B - A$ is positive semi-definite. The spectral norm of $A$ is denoted by ${\norm{A}_2 = \sup_{\norm{x}_2 = 1}\norm{Ax}_2}$. The Moore-Penrose inverse of $A$ is denoted by $A^\dagger$. 
Finally, we use $\mathcal{O}(x)$ to refer to a quantity that depends on $x$ up to a poly-log expression in $d,T$ and~$\delta$, and $\Olog(x)$ represents the leading dependence of $x$ in $d,T$ and $K$. 

\noindent \textbf{Setup. }
Our basic framework consists of sequential interactions of a learner with an environment. We assume the following protocol, which proceeds in discrete trials $t=1, \hdots, T$. At each round $t \in [T]$ the environment outputs a context $x_t \in \X \subseteq \R^d$ sampled from some unknown distribution $\mathcal{P}_x$. We assume that $x_1, \hdots x_T$ are i.i.d. Based on observed payoffs in previous trials, the learner chooses an action $a_t \in \A$, where $\A = [K]$ is the learner's action space. Subsequently, the learner observes a reward $r_t = \inner{x_t, w^*_{a_t}} + \eta_t$, where $\brk[c]*{w^*_{a} \in \R^d}_{a \in \A}$ are unknown parameter vectors, and $\eta_t$ is some conditionally $\sigma$-subgaussian random noise, i.e., for some $\sigma > 0$
\begin{align*}
\E{e^{\lambda \eta_t} | F_{t-1}} \leq \exp\pth{\frac{\lambda^2 \sigma^2}{2}}.
\end{align*}
Here, $\brk[c]*{F_t}_{t=0}^\infty$ is any filtration of $\sigma$-algebras such that for any $t \geq 1$, $x_t$ is $F_{t-1}$-measurable and $\eta_t$ is $F_t$-measurable, e.g., the natural $\sigma$-algebra $F_{t-1} = \sigma \pth{(x_1, a_1, \eta_1), \hdots ,(x_{t-1}, a_{t-1}, \eta_{t-1}), x_t, a_t}$.

The goal of the learner is to maximize the total reward $\sum_{t=1}^T \inner{x_t, w^*_{a_t}}$ accumulated over the course of $T$ rounds. We evaluate the learner against the optimal strategy, which has knowledge of $\brk[c]*{w^*_a \in \R^d}_{a \in \A}$, namely $\pi^*(x) \in \arg\max_{a \in \A} \inner{x, w^*_{a}}$. The difference between the learner and optimal strategy's total reward is known as the regret, and is given by
\vspace{-0.02cm}
\begin{equation*}{
    \Regret{T} 
    =
    \sum_{t=1}^T \inner{x_t, w^*_{\pi^*(x_t)}}
    -
    \sum_{t=1}^T \inner{x_t, w^*_{a_t}}.}
\end{equation*}
In this work we assume to have additional access to a partially observable offline dataset, consisting of partially observable contexts, actions, and rewards. Specifically, we assume a dataset ${\D = \brk[c]*{Q x_i, a_i, r_i}_{i=1}^N}$, in which $\brk[c]*{x_i}_{i=1}^N$ are i.i.d. samples from $\mathcal{P}_x$, $\brk[c]*{a_i}_{i=1}^N$ were generated by some fixed behavior policy, denoted by $\pi_b$, which is a mapping from contexts $x \in \X$ to a probability over actions, and $\brk[c]*{r_i}_{i=1}^N$ were generated by the same model described above. Here, we used $Q \in \R^{L \times d}$ to denote the rectangular matrix
$
Q
=
\begin{pmatrix}
    I_{L} & 0
\end{pmatrix}
$.
That is, without loss of generality, we assume only the first $L$ features of $x_i$ are visible in the data. Throughout our work we will sometimes use the notation $\xk$ and $\xc$ to denote the observed and unobserved (hidden) covariates of $x$, respectively. That is, $x = \begin{pmatrix} (\xk)^T , (\xc)^T \end{pmatrix}^T$, where $\xk \in \R^L,\ \xc \in \R^{d-L}$. 

Notice that the distribution of $\D=\brk[c]*{\xk_i, a_i, r_i}_{i=1}^N$, the partially observable dataset, depends on $\pi_b$. Any statistic we attempt to draw from the offline data depends on the measure induced by $\pi_b$, which we denote by $P^{\pi_b}$ \footnote{More precisely, we define the measure $P^{\pi_b}$ for all Borel sets $R \subseteq [0,1]$, $X \subseteq \X$ and $A \in \A$
$
    P^{\pi_b}(r \in R, x \in X, a \in A)
    =
$ \\
${
    P(r \in R | x \in X, a \in A)
    P(x \in X)
    \int_{x' \in X, a' \in A}
    \indicator{a = a', x = x'} d\pi_b.}
$}. Figure~\ref{fig: setup and projections} depicts a diagram of our basic setup and approach.

\section{From Partially Observable Offline Data to Linear Side Information}
\label{sec: from data to linear constraints}

Consider only having access to the partially observable offline data $\D$. Having access to such data is mostly useless without further assumptions. Particularly, $w^*_a$ may not be identifiable \footnote{We use the notion of identifiability as defined in Definition~2 of \citet{pearl2009causal}}. In fact, it can be shown that for any behavioral policy $\pi_b$ and induced measure $P^{\pi_b}$, $\brk[c]*{w_a^*}_{a \in \A}$ are not identifiable. More specifically, for all ${w^1 = \brk[c]*{w_a^1}_{a \in \A}}$, exist ${w^2 = \brk[c]*{w_a^2}_{a \in \A} \neq w^2}$ and probability measures $P_1, P_2$ such that ${P_1\pth{\xk, a, r ; w^1, \pi_b} = P_2\pth{\xk, a, r ; w^2, \pi_b} }$ and ${\pi_b\pth{a, x; w^1} = \pi_b\pth{a, x; w^2}}$. This claim is a standard type of result. A proof is provided in the supplementary material. 

To mitigate the identification problem, prior knowledge of characteristics of $\brk[c]*{w_a^*}_{a \in \A}$ can be leveraged \citep{cinelli2019sensitivity}. Instead, here we consider access to an online environment, where the covariates that were unobserved in the data are supplied, i.e., fully observed. This enables us to deconfound the data and identify  $\brk[c]*{w_a^*}_{a \in \A}$. 

Prior to constructing our algorithmic approach, we discuss the relation of confounded offline data to partially known linear constraints. This connection is a principal component of our work which enables us to utilize the (possibly not identifiable) partially observable data.

\subsection{Linear Side Information}
In what follows, we show how partially observable data can be reduced to linear constraints of the form $\brk[c]*{M_a w^*_a = b_a, a \in \A}$. Nevertheless $M_a$ will not be identifiable solely from the offline data. More specifically, we specify a low dimensional least squares problem under a model mismatch, showing it converges to a solution with unique structural properties. This will become beneficial in our analysis later on, allowing us to project the linear bandit problem to an approximate lower dimensional subspace, improving performance guarantees.

Let us first consider the case of fully-observable offline data, i.e., $\xk = x$. Here, one would be able (with large amounts of data) to closely estimate $w_a^*$ for all $a \in \A$, using, for example, the linear regression estimator
\begin{align*}
    \hat{w}_a
    =
    \pth{\frac{1}{N_a}\sum_{i = 1}^{N_a} x_ix_i^T}^{-1} 
    \pth{\frac{1}{N_a}\sum_{i = 1}^{N_a} x_i r_i},
\end{align*}
where we denoted $N_a = \sum_{i=1}^N \indicator{a_i = a}$. With $N \to \infty$, under mild assumptions, this estimator would converge to the true weights $w_a^*$ almost surely. 
It is tempting to try and apply a least square estimator to our partially observable data using a lower dimensional model. Particularly, we might try to solve the optimization problem
\begin{align*}
    \min_{b\in \mathbb{R}^L}\sum_{i = 1}^{N_a}\br*{\inner{\xk_i, b} - r_i}^2\ \quad,\forall a\in \A,
\end{align*}
ignoring the fact that $r_i = x_i^T w^*_a + \eta_i$, i.e., that $r_i$ was generated by a higher dimensional linear model. Solving this problem yields
\begin{align}
    b^{LS}_a = \br*{\frac{1}{N_a}\sum_{i = 1}^{N_a} \pth{\xk_i}\pth{\xk_i}^T}^{-1} \br*{\frac{1}{N_a}\sum_{i = 1}^{N_a} \xk r_i}. \label{eq: omega LS}
\end{align}

The following proposition establishes our first main result -- a relation between the lower-dimension least-square estimator $b^{LS}_a$ and the vector $w^*_a$ in the limit of large data $N\rightarrow \infty$ (We discuss the finite data setting in Section~\ref{sec: discussion}).

\begin{restatable}{proposition}{proprestate}[Confoundness $=$ Linear Constraints]
\label{proposition: omega LS to w}
~\\Let $R_{11}(a) = \Eb{\xk\pth{\xk}^T | a}$,  $R_{12}(a) =  \Eb{\xk\pth{\xc }^T | a}$. Assume $R_{11}(a)$ is invertible for all $a\in \A$ \footnote{The invertibility assumption on $R_{11}$ can be verified, since $R_{11}$ can be estimated by statistics of the observable covariates, $\xk$. If it does not hold, other covariates of $\xk$ can be chosen to satisfy this assumption.}. Then, the following holds almost surely for all $a\in \A$.
\begin{equation*}
    \lim\limits_{N \to \infty} b^{LS}_a 
    =
    \brk[r]3{
        I_{L}, \quad R_{11}^{-1}(a)R_{12}(a)
    }
    w^*_a.
\end{equation*}
\end{restatable}

The proof of the proposition is related to regression analysis with misspecified models (see e.g., \citet{griliches1957specification})
and is provided in the supplementary material. It states that, with an infinite amount of data, the low-dimensional least squares estimator in Equation~\eqref{eq: omega LS} converges to a linear transformation of $w_a^*$. This linear transformation depends on the auto-correlation matrix of $\xk$, $R_{11}(a)$, and the cross correlation matrix of $\xk$ and $\xc$, $R_{12}(a)$. While $R_{11}(a)$ can be estimated from the data, $R_{12}(a)$ depends on unseen features of $x$, namely $\xc$, as well as the behavior policy $\pi_b$, and can thus not be approximated from the given data. As such, we will later assume access to a monotonically non-increasing bound of $R_{12}(a)$ for all $a \in \A$. As we discuss in Section~\ref{section: partially observable covariates}, such a bound can be achieved, for example, through queries to $\pi_b$ (i.e., samples $a \sim \pi_b)$.

Proposition~\ref{proposition: omega LS to w} provides us with a structural dependency between $w_a^*$ and the low-order least squares estimator $b_a^{LS}$ that can be calculated from the offline data. Specifically, every $w_a^*$ is constrained to a set
$
    \brk[c]*{w \in \R^d : Mw = b},
$
for some full row rank matrix ${M \in \R^{L \times d}}$ and vector $b \in \R^L$. A natural question arises: How can such linear side information be used? In the next section we show that we can decrease the effective dimensionality of our problem using such linear side information whenever $M$ and $b$ are known exactly. Then, in Section~\ref{section: partially observable covariates}, we expand this result using estimates of the linear relation in Proposition~\ref{proposition: omega LS to w}. We provide improved regret bounds on the linear contextual bandit problem, consequently exploiting the confounded information present in the partially observable data.

\section{Linear Contextual Bandits with Linear Side Information}
\label{sec: side information}

In the previous section we showed how partially observable data can be reduced to linear constraints. Before diving into the subtleties of utilizing the specific structural properties of the linear relations in Proposition~\ref{proposition: omega LS to w}, we form a general result for linear bandits under linear side information when both $M$ and $b$ are given. Particularly, we show that linear side information can be used to improve performance by decreasing the effective dimensionality of the underlying problem.

Assume we are given linear side information
\begin{equation}
\label{eq: side information}
    M_{a} w^*_{a} = b_{a} \qquad, a \in \A.
\end{equation}

In this section we assume $M_{a} \in \mathbb{R}^{L\times d}, b_{a}\in\mathbb{R}^L$ are \textit{known}, and don't assume any structural characteristics. Without loss of generality assume that $\{M_a\}_{a \in \A}$ are full row rank \footnote{If $M_a$ is not full row rank, we remove dependent rows. In fact, we assume $L$ to be the rank of $M_a$.}. One way of using the relations in Equation~\eqref{eq: side information} is by constraining an online learning algorithm to a lower dimensional space. Particularly, notice that for all $a \in \A$,
\begin{align}
\label{eq: side information set}
    w^*_a \in \brk[c]*{w \in \R^d : w = M_a^\dagger b_a + P_a w}, 
\end{align}
where $P_a$ is the orthogonal projection onto the kernel of $M_a$, and is given by
$
    P_a = I - M_a^\dagger M_a.
$ Equation~\eqref{eq: side information set} suggests that knowledge of the linear relation in Equation~\eqref{eq: side information} may allow us to reduce the estimation problem to that of the projected vector, $P_a w_a^*$. Indeed, we may attempt to solve the following corrected, low order ridge regression problem
\begin{align}
    \min_{w\in \mathbb{R}^d}\brk[c]*{\sum_{i=1}^{t-1} \pth{\inner{x_i, P_aw}-y_{a,i}}^2 + \lambda\norm{P_aw}_2^2 },
    \label{eq: prr optimization problem}
\end{align}
where $y_{a,i} =  r_i - \inner{x_i, M_{a_i}^\dagger b_{a_i}}$. Taking its smallest norm solution yields
\begin{align}
    \hat{w}_{t,a}^{P_a}
    =
    &\pth{P_a \pth{\lambda I + \sum_{i=1}^{t-1} x_ix_i^T} P_a}^\dagger \times \nonumber \\
    &\pth{\sum_{i=1}^{t-1} r_i x_i
    -
    \sum_{i=1}^{t-1} x_i x_i^T M_a^\dagger b_a}.
    \label{eq: Pw estimator}
\end{align}

\begin{algorithm}[t!]
\caption{{OFUL with Linear Side Information}}
\label{algo: projected oful}
\begin{algorithmic}[1]
\small
\STATE {\bf input:} $\alpha > 0, M_a \in \R^{L\times d}, b_a \in \R^L, \delta>0$
\STATE {\bf init:} $V_a = \lambda I_d, Y_a = 0, \forall a \in \A$
\FOR{$t = 1, \hdots$}
    \STATE  Receive context $x_t$
    \STATE  $\hat{w}_{t,a}^{P_a} = \pth{P_a V_a P_a}^\dagger
    \pth{Y_a - \pth{V_a - \lambda I_d} M_a^\dagger b_a}$
    \STATE ${\hat{y}_{t,a} =
    \inner{x_t, M_a^\dagger b_a} 
    + 
    \inner{x_t, \hat{w}_{t,a}^{P_a}}}$
    \STATE $\text{UCB}_{t,a} = \sqrt{\beta_{t} (\delta)} \norm{x_t}_{(P_aV_aP_a)^\dagger}$
    \STATE $a_t\in \arg\max_{a\in \A}\brk[c]*{\hat{y}_{t,a} + \alpha \text{UCB}_{t,a}}$
    \STATE Play action $a_t$ and receive reward $r_t$
   \STATE $V_{a_t} = V_{a_t} + x_tx_t^T, Y_{a_t} = Y_{a_t} + x_t r_t$
\ENDFOR
\end{algorithmic}
\end{algorithm}

Perhaps intuitively, this least squares estimator is in fact equivalent to one in a lower dimensional space~$\R^m$, the rank of $P_a$. Indeed, letting $P_a = UU^T$, where $U \in \R^{d \times m}$ is a matrix with orthonormal columns\footnote{As orthogonal projection matrices have eigenvalues which are either 0 or 1, any projection matrix can be decomposed into $P = UU^T$, where $U$ is a matrix with $\rank(P)$ orthonormal columns.}, we have that (see supplementary material for full derivation)
\begin{align*}
    U^T \hat{w}_{t,a}^{P_a} 
    = 
    &\pth{\lambda I_m + \sum_{i=1}^{t-1} \pth{U^T x_i} \pth{U^T x_i}^T}^{-1} \times \\
    &\pth{\sum_{i=1}^{t-1} y_{a,i} \pth{U^T x_i}}.
\end{align*}
That is, $U^T\hat{w}_{t,a}^{P_a}$ is a least squares estimator in~$\mathbb{R}^m$. 

We are now ready to construct a least squares variant for $w^*_a$, which utilizes the information in Equation~\eqref{eq: side information}. Having an estimation for $P_a w^*_a$, we make use of the set defined in Equation~\eqref{eq: side information set} to construct our final estimator
$
    \hat w_{a,t} = M^\dagger_a b_a + \hat{w}_{t,a}^{P_a},
$
where $\hat w^{P_a}_{t,a}$ is given by Equation~\eqref{eq: Pw estimator}. Then, estimation of $\hat w_{a,t}$ will depend on the rank of $P_a$, i.e., $\rank(P_a) = d-L$. In what follows we will show how this projected estimator can be integrated into a linear bandit algorithm, reducing its effective dimensionality to that of the rank of $P_a$, i.e., $d-L$.



Algorithm~\ref{algo: projected oful} describes the reduction of the OFUL algorithm \citep{abbasi2011improved} to its projected variant, in which linear side information is leveraged by means of low order ridge regression (Equations~\eqref{eq: prr optimization problem}) to decrease the effective dimensionality of the problem. In Line~5 of the algorithm, the estimator of Equation~\eqref{eq: Pw estimator} for $P_a w^*_a$ is used. This becomes useful in Line~7, as the confidence set around $w^*_a$ is reduced to a lower dimension, i.e., $d-L$.

For all $a \in \A$, assume $\norm{P_a x_i}_2 \leq \xoBound$ almost surely and $\norm{P_a w_a^*}_2\leq \woBound$. Letting ${\sqrt{\beta_t(\delta)} = \lambda^{1/2}\woBound + \sigma \sqrt{(d-L) \log{\frac{K(1+t\xoBound^2/\lambda)}{\delta}}}}$, the following theorem provides the improved regret of Algorithm~\ref{algo: projected oful}. Its proof is given in the supplementary material, and is based on a reduction of the linear bandit problem to a lower dimensional space, based on Equation~\eqref{eq: Pw estimator}.

\begin{theorem}
\label{thm: regret bound oful}
For all $T\geq 0$, with probability at least $1-\delta$, the regret of Algorithm~\ref{algo: projected oful} is bounded by
\begin{align*}
    Regret{T}\leq \Olog\br*{(d-L)\sqrt{KT}}.
\end{align*}
\end{theorem}


Indeed, by Theorem~\ref{thm: regret bound oful}, linear relations of rank $L$ reduce the linear bandit problem to a lower dimensional problem, with regret guarantees that are equivalent to those of a linear bandit problem of dimension $d-L$. However, these results hold only for $M_a, b_a$ that are fully known. When $\brk[c]*{M_a}_{a \in \A}$ are unknown, we must rely on estimations of $M_a$. The accuracy of our estimation as well as its rate of convergence would highly affect the applicability of such constraints. As we will see next, the linear transformation of Proposition~\ref{proposition: omega LS to w} can be efficiently estimated whenever $R_{12}$ can be efficiently estimated. Such an assumption will allow us to achieve similar regret guarantees under mild conditions.

\begin{algorithm}[t!]
\caption{{OFUL with Partially Observable           Offline Data}}
\label{algo: oful with doubling}
\begin{algorithmic}[1]
\small
\STATE {\bf input:} $\alpha\! >\! 0,\delta\!>\!0, T$, $b_a\!\in\! \mathbb{R}^L$ (from dataset)
\FOR{$n = 0, \hdots, \lg{T}-1$}
    \STATE Use $2^n$ previous samples from $\pi_b$ to 
    \STATEx ~~~~update the estimate of $\pert{M}_{2^n,a}, \forall a\in \A$ 
    \STATE Calculate $\pert{M}_{2^n,a}^\dagger,\pert{P}_{2^n,a}$, $\forall a\in \A$
    \STATE Run Algorithm~\ref{algo: projected oful} for $2^n$ time steps with bonus
    \STATEx ~~~~$\sqrt{\beta_{n,t}(\delta)}$ and  $\pert{M}_{2^n,a}, b_a$
\ENDFOR
\end{algorithmic}
\end{algorithm}

\section{Deconfounding Partially Observable Data}
\label{section: partially observable covariates}

This section builds upon the observations collected in the previous sections in order to construct our second main result: an algorithm that leverages large, partially observable, offline data in the online linear bandit setting. While Proposition~\ref{proposition: omega LS to w} seemingly provides us with linear side information in the form of linear equalities $M_a w^* = b_a$, the matrix $M_a$ cannot be obtained from the partially observable offline data, since $R_{12}(a)$ depends on the unobserved covariates $\xc$, as well as the behavior policy $\pi_b$. Nevertheless $M_a = \brk[r]*{
        I_{L}, \quad R_{11}^{-1}(a)R_{12}(a)
}$ can be efficiently estimated whenever $R_{12}(a) = \Eb{\xk\pth{\xc }^T | a}$ can be efficiently estimated. Particularly we make the following assumption.
\begin{assumption}
\label{assumption: query pib}
    We assume for every $t > 0$ we can approximate $R_{12}(a), \forall a \in. \A$ such that 
    \begin{align*}
    \norm{R_{12}(a) - \hat{R}_{12}(a, t)}_2 \leq \frac{g(d, L)}{\sqrt{t}} \quad\text{w.h.p.}
    \end{align*}
\end{assumption}

\subsection{Case Study: Queries to $\pi_b$} Consider the problem of identifying the statistic $R_{12}(a)$. Due to its dependence on $\pi_b$, this may be impossible without access to $\pi_b$ or other information on its induced measure,~$P^{\pi_b}$. As such, we assume that during online interactions, the online learner can query~$\pi_b$, i.e., sample an action $a^b \sim \pi_b(x)$. 

Having access to queries from $\pi_b$, we can construct an online estimator for the cross-correlation matrix $R_{12}(a)$. More specifically, at each round $t \in [T]$, we observe a context $x_t$ and query $\pi_b$ by sampling $a_t^b \sim \pi_b(x_t)$. We then estimate $R_{12}(a)$ using the empirical estimator\footnote{In fact, we can construct a tighter estimator for $R_{12}(a)$ using our knowledge of $\Eb{\xk|a}$, which can be estimated exactly from the offline data. We leave its analysis out for clarity.}
\begin{align*}
    \hat{R}_{12}(a, t)
    =
    \frac{1}{t}
    \sum_{i=1}^t
    \frac{\indicator{a_i = a}}
         {P^{\pi_b}(a)}
    \pth{\xk_i} \pth{\xc_i}^T,
\end{align*}
where $P^{\pi_b}(a)$ is known due to the offline data. Assuming $\norm{\xk}_2 \leq S_1$ and $\norm{\xc}_2\leq S_2$ a.s., it can be shown that with probability at least $1-\delta$ (see supplementary material, Lemma~\ref{thm: masked CC estimation}, for proof)
\begin{align*}
    &\norm{R_{12}(a) - \hat{R}_{12}(a, t)}_2 
    \leq \\
    &
    \mathcal{O}\pth{
    S_1S_2
    \sqrt{
        \frac{1}{t} 
        \pth{\frac{\sqrt{\trace{R_{11}}\trace{R_{22}}}}{S_1S_2}}
        \log{\frac{d}{\delta}}
    }},
\end{align*}
indeed, satisfying Assumption~\ref{assumption: query pib}. We can now naturally construct an estimator for $M_a$. Its estimator is given by
\begin{align}
\label{eq: Ma estimator}
    \hat{M}_{t,a} 
    = 
    \brk[r]3{
        I_{L}, \quad R_{11}^{-1}(a)\hat{R}_{12}(a, t)
    }.
\end{align}

A natural question arises: can the estimated linear constraints $\hat{M}_a w^*_a = b_a$ be used as linear side information while still maintaining the regret guarantees of Theorem~\ref{thm: regret bound oful}, i.e., decrease the effective dimensionality of the problems from $d$ to $d-L$? Specifically, we wish to construct a variant of Algorithm~\ref{algo: projected oful} in which $\hat{M}_{t,a}$ are used as linear side information. In this setting the estimated projection matrix $\hat{P}_{t,a}$ and the estimated Moore-Pensore Inverse $\hat{M}^\dagger_{t,a}$ are directly calculated from $\hat{M}_{t,a}$, i.e., these matrices are approximate. 

\begin{figure*}[t!]
\captionsetup[subfigure]{labelformat=empty}
\begin{subfigure}{.32\textwidth}
    \centering
    \includegraphics[width=\textwidth]{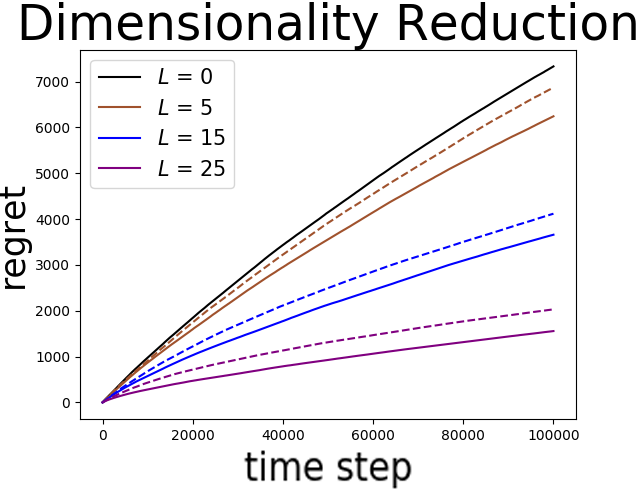}
    \caption{(a)}
\end{subfigure}
\begin{subfigure}{.32\textwidth}
    \centering
    \includegraphics[width=\textwidth]{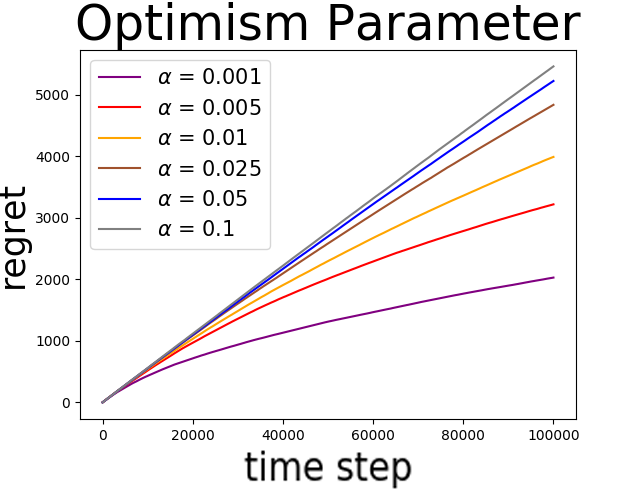}
    \caption{(b)}
\end{subfigure}
\begin{subfigure}{.32\textwidth}
    \centering
    \includegraphics[width=\textwidth]{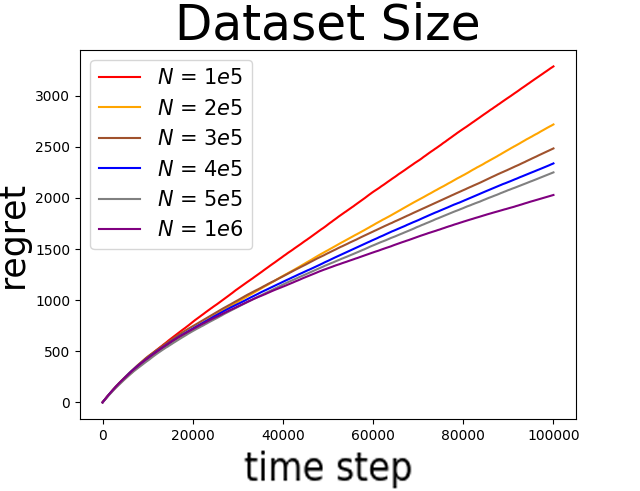}
    \caption{(c)}
\end{subfigure}
\centering
\caption{\label{fig: oful}  \footnotesize  All experiments were conducted with the same vectors $w^*_a$ of dimension $d=30$ and $K=30$ arms. \textbf{(a)} Plot compares effect of $L$ when $R_{12}$ is known (solid lines) vs. estimated (dashed lines). For $L=0$ (i.e., no side information) we executed Algorithm~\ref{algo: projected oful} without using the dataset. \textbf{(b)} Comparison of different values of~$\alpha$ using an offilne dataset and $L=25$. \textbf{(c)} Effect of dataset size on performance for $L=25$. }
\end{figure*} 

Algorithm~\ref{algo: oful with doubling} describes the linear bandit variant with partially observable confounded data. Note that, unlike Algorithm~\ref{algo: projected oful}, Algorithm~\ref{algo: oful with doubling} is not an anytime algorithm, but rather acts knowing the horizon $T$. Assuming $\norm{x_i}_2 \leq \xBound$ a.s. and $\norm{w_a^*}_2\leq \wBound$ for all $a \in \A$, the algorithm uses an augmented confidence, given by ${\sqrt{\beta_{n,t}(\delta)} = \lambda^{1/2}\wBound + \br*{\sigma+\xBound \wBound f_n} \sqrt{(d-L) \log{\frac{1+t\xBound^2/\lambda}{\delta/2\log{T}K}}}}$, where ${f_n= f_{B1} + f_{B2} 2^{-n/2}}$, $f_{B1} = \Olog\br*{ \max_{a}\frac{\lambda_{\mathrm{min}}\br*{R_{11}(a)}^{-1}}{P^{\pi_b}(a)}\xBound\br*{\trace{R_{11}(a)}\trace{R_{22}(a)}}^{1/4}}$ and $f_{B2} = \Olog\br*{\max_{a}\frac{\lambda_{\mathrm{min}}\br*{R_{11}(a)}^{-1}}{P^{\pi_b}(a)}\xBound^2}$. At every time step~$t\in[T]$, the learner uses the estimate $\hat{M}_{t,a}$ and subsequently considers it to be linear side information, as in Algorithm~\ref{algo: projected oful}. The following theorem provides regret guarantees for Algorithm~\ref{algo: oful with doubling}, proving partially observable data can be beneficial for online learning.

\begin{theorem}
\label{thm: oful with data}
    For any $T > 0$, with probability at least $1-\delta$, the regret of Algorithm~\ref{algo: oful with doubling} with the estimator given in Equation~\eqref{eq: Ma estimator} is bounded by
    \begin{align*}
    &\text{Regret}(T) 
    \leq
    3\sqrt{T (d-L)K  \log{\lambda+ \frac{T\xBound^2}{d-L}}} \times \\ 
    &\br*{\sigma_\epsilon \sqrt{(d-L) \log{\frac{1+T\xBound^2/\lambda}{\delta/(2K\log{T})}}} + \lambda^{1/2}\wBound}
    + \\
    &\mathcal{O}\pth{(d-L)\sqrt{K}\xBound\wBound f_{B2}}
    ,
    \end{align*}
    where $\epsilon = \xBound \wBound f_{B1}$ and $\sigma_\epsilon = \sigma+\epsilon$.
    This leads to, ${\Regret{T}\leq \Olog\br*{ (1 + f_{B_1})(d-L)\sqrt{KT} }.}$
\end{theorem}
Notice that, unlike Theorem~\ref{thm: regret bound oful}, the regret of Algorithm~\ref{algo: oful with doubling} is worsened asymptotically by a factor relating to $f_{B_1}$. This function can also scale with $d$, due to its dependence on $\pth{\trace{R_{11}(a)}}^{1/4}$ and $\pth{\trace{R_{22}(a)}}^{1/4}$. Specifically, a worst case dependence yields ${f_{B_1} \leq \Olog\br*{\max_{a}\frac{\pth{L(d-L)}^{1/4}}{P^{\pi_b}(a)}}}$, where here $\max_{a} \frac{1}{P^{\pi_b}(a)} \geq K$. That is, $f_{B_1}$ is a factor indicating how hard it is to approximate the linear constraints, dependent on the amount of information in $x$ as well as the support of the behavior policy, $\pi_b$. Still, in settings in which $d$ and $T$ are prominent over $K$, a significant improvement in performance is achieved.

The proof of the theorem is provided in the supplementary material. Unlike in Theorem~\ref{thm: regret bound oful}, we do not have access to the true matrices $M^\dagger_a,P_a$, but to increasingly more accurate estimates of these matrices. To deal with this more challenging situation we use the doubling trick. The algorithm acts in exponentially increasing episodes. In each such episode, we fix the estimation of $M_a$, i.e., we use the estimate of $M_a$ available at the beginning of the episode. The analysis of this algorithm amounts to study the performance of the exact algorithm (as in Theorem~\ref{thm: regret bound oful}) up to a fixed, approximated, $M_a$, which induces errors in the used $M^\dagger_a,P_a$. Finally, summing the regret on each episode, we obtain the result. 

The proof heavily relies on the convergence properties of $P_a, M^\dagger$, which are shown to converge at a rate of $O(T^{-1/2})$. These convergence rates are due to the special structure of $M_a$. Specifically, we prove that ${\norm{P_a-\hat{P}_{t,a}}\leq 2\norm{M_a-\hat{M}_{t,a}}}$ and ${ \norm{M_a^\dagger-\hat{M}_{t,a}^\dagger}\leq 2\norm{M_a-\hat{M}_{t,a}}}$,  meaning, the convergence of $\hat{P}_{t,a}$ and $\hat{M}^\dagger_{t,a}$ is well controlled by the convergence of $\hat{M}_{t,a}$. This property does not hold for general matrices. In fact, for a general matrix $A$, $A^\dagger$ is not even continuous w.r.t. perturbations in $A$ (see e.g.,~\citealt{stewart1969continuity}).~Thus, the structure of $M_a$ establishes convergence rates of $\hat{P}_{t,a}, \hat{M}^\dagger_{t,a}$ sufficient to achieve the desired regret.

Algorithm~\ref{algo: oful with doubling} is highly wasteful w.r.t. the information gathered through time. Specifically, it discards all information upon updates of $\hat{M}_{t,a}$. In a practical setting, we  expect the algorithm to achieve similar performance guarantees even when information is not discarded. Moreover, as we show empirically in the next section, significant improvement can still be achieved without applying the doubling trick, i.e., by running Algorithm,~\ref{algo: projected oful} with the approximated $\hat{M}_{t,a}$.


\section{Experiments}
\label{sec: experiments}

In this section we demonstrate the effectiveness of using offline data in a synthetic environment. Our environment consisted of $K=30$ arms and vectors $w^*_a \in \R^{30}$ uniformly sampled in $\brk[s]*{0, \frac{1}{d}}^d$ and fixed across all experiments. Contexts were sampled from a uniform distribution in $\brk[s]*{0, 1}^d$ and normalized to have norm $1$. The behavioral policy $\pi_b$ was chosen to follow a softmax distribution $\pi_b(a,x) \propto \exp\pth{\phi_a^T x}$, where $\phi_a \in \R^d$ were randomly chosen and fixed across all experiments.

Figure~\ref{fig: oful}a illustrates the effectiveness of using partially observable data. We used a dataset of $1$~million examples to simulate a sufficiently large dataset. Solid lines depict regret when $R_{12}(a)$ are known in advance, allowing us to apply Algorithm~\ref{algo: projected oful} without estimations (Section~\ref{sec: side information}). Dashed lines depict regret for the estimated case using queries to $\pi_b$, i.e., $M_a$ were estimated at every iteration using an estimate of $R_{12}(a)$ (see Section~\ref{section: partially observable covariates}). Note that $L=0$ corresponds to the linear bandit problem with no side information, i.e., the original OFUL algorithm. It is evident that utilizing the partially observable data can significantly improve performance, even when using approximate projections. We note that the experiments were run under constant updates of $\hat{M}_{t,a}$, i.e., without epoch schedules.

Figure~\ref{fig: oful}b depicts the effect of the optimism parameter $\alpha$ (see Algorithm~\ref{algo: projected oful}) on overall performance when utilizing a dataset with $L=25$ observed features. A gap is evident between the proposed theoretical confidence and the practical results, as very small values of $\alpha$ showed best performance. This gap is most likely due to worst case scenarios that were not imposed by our simulated environments. 

Finally, Figure~\ref{fig: oful}c depicts experiments with varying amount of data. While the number of examples has an effect, it does not significantly deteriorate overall performance, suggesting that partially observable offline data can be used even with finite datasets, as long as they are sufficiently large.

\section{Related Work}
\label{sec: related work}

 The linear bandits problem, first introduced by \citet{auer2002using}, has been extensively investigated in the pure online setting \citep{dani2008stochastic,rusmevichientong2010linearly,abbasi2011improved}, with numerous variants and extensions
 \citep{agrawal2016linear,kazerouni2017conservative,amani2019linear}.
 
The offline (logged) bandit setting usually assumes the algorithm must learn a policy from a batch of fully observable data \citep{shivaswamy2012multi,swaminathan2015batch,joachims2018deep}.
The use of offline data has also been investigated under the reinforcement learning framework, including batch-mode off-policy reinforcement learning and off-policy evaluation
\citep{ernst2005tree,lizotte2012linear,fonteneau2013batch,precup2000eligibility,thomas2016data,gottesman2019combining}

More related to our work are attempts to establish unbiased estimates or control schemes from confounded offline data \citep{lattimore2016causal,oberst2019counterfactual,tennenholtz2019off}. Other work in which partially observable data is used usually consider the standard confounded setting (e.g., identification of $P(r | \text{do}(a))$) \citep{zhang2019near,ye2020combining}. \citet{wang2016learning} also consider hidden features, where biases are accounted for under assumptions on the hidden features. In these works the unobserved features (confounders) are never disclosed to the learner. Prior knowledge is thereby usually assumed over their support (e.g., known bounds). When such priors are unknown, these methods may thus fail. Moreover, they are sub-optimal in settings of fully observable interactions, where unobserved confounders become observed covariates.

In this work we view the problem from an online learner's perspective, where \emph{offline data is used as side information}. Specifically we project the given information, reducing our problem to its orthogonal subspace. Projections have been previously used in the bandit setting for reducing time complexity and dimensionality \citep{yu2017cbrap}. Other work consider bandits under constraints \citep{agrawal2019bandits}. Finally, \citet{djolonga2013high} consider subspace-learning by combining Gaussian Process UCB sampling and low-rank matrix recovery techniques.

\section{Discussion and Future Work}
\label{sec: discussion}

In this work we showed that partially observable confounded data can be efficiently utilized in the linear bandit setting. In this section we further discuss two central assumptions made in our work; namely, infinite data and bounding the cross correlation matrix $R_{12}$.

\noindent \textbf{Finite Data. }
Throughout our work we assumed the limit of infinite sized data. From a technical perspective, the use of finite data would introduce an error in the least squares estimator \citep{krikheli2018finite}. A straightforward analysis would propagate this error as additional linear penalty to the regret that is dependent on the number of samples in the data. More involved techniques may combine optimistic bounds on the finite samples in the data. We chose to leave its derivation out to focus on the topic of missing covariates in the data. Finally, our experiments demonstrate that the number of samples does not greatly affect performance, as long as they are sufficiently large, i.e., when the error is small relative to $T$.

\noindent \textbf{Bounding $R_{12}$. } 
Being able to estimate $R_{12}(a)$ is an essential requirement for deconfounding the partially observable data. Nevertheless, $R_{12}(a)$ is dependent on $\pi_b$, raising the question, can $R_{12}(a)$ be estimated without knowledge of $\pi_b$? In our work we showed how one can estimate it using queries to $\pi_b$. In fact, we did not require knowledge of $\pi_b$, nor did we require interactions of $\pi_b$ with the environment (i.e., we do not act according to $\pi_b$), but rather, only view samples from $\pi_b$. While such an assumption may be strict in some settings, it is reasonable in others. For instance, when $\pi_b$ was controlled by us when the data was recorded. Other settings for estimating $R_{12}(a)$ are also possible, e.g., having access to additional fully observable datasets that were generated by $\pi_b$ \citep{kallus2018removing}. 

Consider the examples of the healthcare and traffic management settings presented in Section~\ref{sec: intro}. In the medical setting, quering $\pi_b$ would amount to asking the clinician that induced the data what she would have done in a provided situation. In this scenario, cooperation of the clinician is needed to deconfound the data. Nevertheless, note that this approach is not limited by the amount of confounding bias inherent in the data, allowing us identify \emph{optimal} control policies. Unlike the medical example, in the traffic management example we have access to the behavior policy that generated the data. In this scenario, the querying assumption is insignificant.

\textbf{Future Work.} While this work assumed a monotonically vanishing error of $\hat{R}_{12}(a)$ (i.e., asymptotic identifiability), future work can consider looser bounds on the estimate. It is also interesting to understand the  contextual bandit algorithms, both in the linear as well as the general function class settings. It is also interesting to generalize our results to the reinforcement learning setting.

\begin{acknowledgements} This research was partially supported by the ISF under contract 2199/20.
\end{acknowledgements}

\bibliography{citations.bib}

\newpage

\onecolumn
\appendix
\appendixtitleon
\appendixtitletocon

\listofappendices

\begin{appendices}

\section{Overview}

The appendix is orginized as follows:

\begin{itemize}
\item Appendix~\ref{supp: moore penrose inverse} provides a short background on orthogonal projections and the Moore-Penrose inverse, a fundamental tool in our work. We also state some well known pertrubations bounds that will be useful for achieving well-behaved convergence rates in the proof of Theorem~\ref{thm: oful with data}. 

\item In Appendix~\ref{supp: impossibility} we discuss the unidentifiability problem of using partially observable data, with focus to our setting. Specifically, we show that $\brk[c]*{w_a^*}_{a \in \A}$ are not identifiable, and provide further motivation for this work. 

\item Appendix~\ref{supp: proof of prop 1} provides a proof of Proposition~\ref{proposition: omega LS to w}, namely, showing that the structure of the transformation obtained from the offline data is given by $M_a = (I_L, R_{11}(a)^{-1}R_{12}(a))$.

\item In Appendix~\ref{supp: properties of prr} we discuss Projected Ridge Regression (P-RR) with side information, as presented in Section~\ref{sec: side information}. Specifically, we state and prove the equivalence of P-RR to a low-order ridge regression problem with a correction term (see Equation~\eqref{eq: prr optimization problem}). 

\item Appendix~\ref{supp: oful} provides a proof of Theorem~\ref{thm: regret bound oful}, showing that Algorithm~\ref{algo: projected oful} reduces the OFUL algorithm to the lower-dimensional projected subspace of known linear transformation $\brk[c]*{M_a w = b_a}$, decreasing the regret from $\Olog\pth{d\sqrt{KT}}$ to $\Olog\pth{(d-L)\sqrt{KT}}$. Its proof is a direct consequence of the P-RR formulation.

\item Based on the above, we formulate our final result, proving Theorem~\ref{thm: oful with data} in Appendix~\ref{supp: approximate oful}. We begin by showing the approximate linear transformations obtained by estimating $R_{12}(a)$ from online samples are well behaved. Specifically, we show that the projection and pseudo-inverse operators converge at the same rate as $\hat{M}_{t,a}$, i.e., $\norm{P_a- \pert{P}_{t,a}}\leq 2\norm{M_a-\pert{M}_{t,a}}$ and $\norm{M_a^\dagger - \pert{M}_{t,a}^\dagger}\leq 2\norm{M_a-\pert{M}_{t,a}}$. We then leverage these important properties, together with a doubling trick approach, showing that similar regret guarantees can be achieved as in the exact case.
\end{itemize}

\newpage
\section{Orthogonal Projections and the Moore–Penrose Inverse}
\label{supp: moore penrose inverse} 

In this section we give a short review of the Moore–Penrose inverse \citep{barata2012moore} and its corresponding orthogonal projection. We state some well known properties that will be useful in our analysis. 

For a matrix $M \in \R^{L \times d}$ we denote its Moore–Penrose inverse by $M^\dagger$. Let $P^{\parallelsum} = M^\dagger M$ be the orthogonal projection onto the range of $M$ and $P^\bot = I - P^{\parallelsum}$ be the orthogonal projection onto the kernel of $M$. We have the following well known properties.
\begin{property}
\label{property 1}
    If $M$ has independent rows, then $M^\dagger$ can be computed as
    $
        M^\dagger = M^T(MM^T)^{-1}.
    $
\end{property}
\begin{property}
\label{property 2}
    If $Mw = b$ then $P^{\parallelsum}w = M^\dagger b$.
\end{property}
\begin{property}[\citealt{planitz19793}]
\label{property 3}
    The vector $x = M^\dagger b$ is the vector with the smallest $L_2$ norm which satisfies $Mx=b$.
\end{property}
\begin{property}\label{property: linear side information}
If $w\in \mathbb{R}^d$ satisfies $Mw=b$ then $w$ can be written as
\begin{align*}
    w = P^{\bot}w + M^\dagger b.
\end{align*}
\end{property}
\begin{proof}
    Using the fact $I = P^{\parallelsum}w  + P^{\bot}$ and Property~\ref{property 2} we get
    \begin{align*}
        w = P^{\bot}w + P^{\parallelsum}w = P^{\bot}w + M^\dagger b.
    \end{align*}
\end{proof}

\begin{remark}[Notation]
For brevity, we denote $P= P^{\bot}$ as the orthogonal projection to the kernel of $M$ and $I-P=I-P^{\bot} =  P^{\parallelsum}$ as the orthogonal projection to the range of $M$.  
\end{remark}

\subsection{Useful Perturbation Bounds}

The following two results are standard in perturbation theory. They bound the $L_2$ difference between some matrix $A$, and its perturbed counterpart $B = A + E$, where $E$ a perturbation (i.e., error) matrix. 

\begin{theorem}[E.g., \citealt{chen2016perturbation}, Corollary 2.7] \label{theorem: P under perturbation} For any matrices let $A,B,E\in \mathbb{R}^{d\times d}$ and $E=B-A$. Let $P_A$ and $P_B$ be the orthogonal projection on the raw space of $A$ and $B$, respectively. Assume $\rank(A)=\rank(B)$. Then,
\begin{align*}
    \norm{P_A - P_B}_2 \leq \min\br*{\norm{A^\dagger}_2,\norm{B^\dagger}_2}\norm{E}_2,
\end{align*}
where $P,\pert{P}$ are the orthogonal projections into the row space of $M, \pert{M}$, respectively,
\end{theorem}

\begin{theorem}[\citealt{stewart1977perturbation}, Theorem 3.3] \label{theorem: M dagger perturbabation without equal rank} For any matrices let $A,B,E\in \mathbb{R}^{d\times d}$ and $E=B-A$. Then,
\begin{align*}
    \norm{B^\dagger - A^\dagger}_2 \leq 2\max\br*{\norm{A^\dagger}^2_2,\norm{B^\dagger}^2_2}\norm{E}_2.
\end{align*}
\end{theorem}

\begin{remark}
Note that Theorem~\ref{theorem: P under perturbation} assumes $\rank(A)=\rank(B)$. While other perturbation bounds exist for the case $\rank(A)\neq\rank(B)$, they do not provide sufficient guarantees for our analysis (e.g., $\norm{A^\dagger - B^\dagger}$ may diverge~\citealt{stewart1969continuity}). Luckily, due to the special structure of $M$, i.e., $M = (I_L, R_{11}^{-1}R_{12})$, the perturbed version of $M$ as defined in Section~\ref{section: partially observable covariates} will always have rank $L$, ensuring this assumption holds.
\end{remark}

\newpage
\section{Unidentifiability Problem of Partially Observable Data}
\label{supp: impossibility}

Having access to partially observable offline data may not be enough to obtain the optimal policy, even if $N \to \infty$. The following is a standard identifiability result, showing that $\brk[c]*{w_a^*}_{a \in \A}$ are not identifiable in the setting of partially observable data and no online interaction.

\begin{proposition}
For any behavioral policy $\pi_b$ and induced measure $P^{\pi_b}$, $\brk[c]*{w_a^*}_{a \in \A}$ are not identifiable. More specifically, for all $w^1 = \brk[c]*{w_a^1}_{a \in \A}$ exist $w^2 = \brk[c]*{w_a^2}_{a \in \A} \neq w^2$ and probability measures $P_1, P_2$ such that ${P_1\pth{\xk, a, r ; w^1, \pi_b} = P_2\pth{\xk, a, r ; w^2, \pi_b} }$ and $\pi_b\pth{a, x; w^1} = \pi_b\pth{a, x; w^2}$.
\end{proposition}

\begin{proof}
    Denote by $w^o_a \in \R^L, w^h_a \in \R^{d-L}$, the vectors corresponding to the observed and hidden parts of $x$, namely, $\xk$ and $\xc$, respectively. That is, $w_a = \begin{pmatrix} w^o_a \\ w^h_a\end{pmatrix}$ and 
    \begin{align*}
        r_i = x_i^T w_a + \eta_i =  (\xk)^T w^o_a + (\xc)^T w^h_a + \eta_i.
    \end{align*}
    Let $P_1, P_2$ be two measures such that $P_1((\xc)^T w^{1o}_a \leq \alpha | a, \xk) = P_2((\xc)^Tw^{2o}_a \leq \alpha| a, \xk)$ for all $\alpha \in \mathbb{R}$. Note that such measures always exist. For example letting $\xc$ be a random vector with i.i.d. elements independent of $\xk$, such that
    \begin{align*}
    \begin{cases}
        &P_1\pth{(\xc)_k = \frac{1}{(w^{1o}_a)_k}}
        =
        \frac{1}{2} \\
        &P_1\pth{(\xc)_k = \frac{1}{(w^{2o}_a)_k}}
        =
        0 \\
        &P_1\pth{\xc)_k = 0}
        =
        \frac{1}{2}
    \end{cases} \qquad
    \begin{cases}
        &P_2\pth{(\xc)_k = \frac{1}{(w^{1o}_a)_k}}
        =
        0 \\
        &P_2\pth{(\xc)_k = \frac{1}{(w^{2o}_a)_k}}
        =
        \frac{1}{2} \\
        &P_2\pth{\xc)_k = 0}
        =
        \frac{1}{2}
    \end{cases}
    \end{align*}
    where without loss of generality we assumed $(w^{1o}_a)_k \neq 0, (w^{2o}_a)_k \neq 0, \forall 1 \leq k \leq d-L$ \footnote{If one of them is zero, simply choose $P_1\pth{(\xc)_k = 0} = P_2\pth{(\xc)_k = 0} = 1$}. Then it follows that $P_1((\xc)^T w^{1o}_a \leq \alpha | a, \xk) = P_2((\xc)^T w^{2o}_a \leq \alpha | a, \xk)$.
    
    We have that
    \begin{align*}
        P_1(r \leq \beta | a = a, \xk = c ; w^1, \pi_b)
        &=
        P_1((\xc)^T w^o_a + \eta \leq \beta - c^Tw^o_a | a = a, \xk = c; w^1, \pi_b) \\
        &=
        P_2(r \leq \beta | a = a, \xk = c ; w^2, \pi_b).
    \end{align*}
    Therefore,
    \begin{align*}
        P_1(r, a, \xk ; w^1, \pi_b) 
        &=
        P_1(r | a, \xk ; w^1, \pi_b)P^{\pi_b}(a, \xk) \\
        &=
        P_2(r | a, \xk ; w^2, \pi_b)P^{\pi_b}(a, \xk) \\
        &=
        P_2(r, a, \xk ; w^2, \pi_b).
    \end{align*}
\end{proof}

The above proposition tells us that partially observable offline data cannot be used unless further assumptions are made. This is true even in the linear model, which is the focus of this work. To mitigate this problem, prior knowledge and assumptions over the hidden variables can be used. Nevertheless, such assumptions may be inaccurate and impossible to validate. Moreover, most concurrent assumptions (such as bounding the ``amount of confoundness"), do not completely resolve this issue, but rather mitigate it so that perhaps better policies can be found. 

This work considers an alternative assumption, namely, access to online interactions with the environment. This has several benefits over assuming prior knowledge:
\begin{enumerate}
    \item There is no problem of validating this assumption, i.e., if we do not have access to an online environment, we will know it.
    \item The online regime allows us to achieve an optimal policy. Specifically, since online interactions reveal the hidden context, an optimal policy is always identifiable.
    \item Partially observable data cannot hurt us, but rather only improve our performance. Looking at the problem from an online learner's point of view, the offline data is only used to improve its performance. Therefore, the offline data does not bias our results, but only decreases the number of online interactions.
\end{enumerate}

\newpage
\section{Proof of Proposition~1}
\label{supp: proof of prop 1}

We recall Proposition~\ref{proposition: omega LS to w}
\proprestate*

Before proving the proposition, we remind the reader of the continuous mapping theorem, which states that continuous functions preserve limits even if their arguments are sequences of random variables.

\begin{theorem}[Continuous Mapping, e.g.,~\citep{van2000asymptotic}]\label{theorem: CMT} Let $\brc*{X_n}$ be a set of real random variables such $X_n\in\mathbb{R}^k$ and $X_n\overset{a.s}{\rightarrow} X$. Let $d(\cdot,\cdot): \mathbb{R}^d\times \mathbb{R}^d \rightarrow \mathbb{R}$ be a distance function that generates the usual topology. Let $g:\mathbb{R}^k \rightarrow \mathbb{R}^m$ be a continuous function at the point $X\in \mathbb{R}^k.$ Then $g(X_n)\overset{a.s}{\rightarrow} g(X)$.
\end{theorem}
\begin{proof}[Proof of Proposition~\ref{proposition: omega LS to w}]
    

    Fix $a\in \A$. By definition, $b^{LS}_a$ is also given by~\eqref{eq: omega LS},
    \begin{align}
    b^{LS}_a &= \br*{\frac{1}{N_a}\sum_{n\in [N]} \indicator{a_n=a} \xk_n \pth{\xk_n}^T }^{-1} \br*{\frac{1}{N_a}\sum_{n\in [N]}\indicator{a_n=a} \xk_n r_n},
    \end{align}
    
Define $\pth{\Qk}^T=\begin{pmatrix}
I_L & 0 \end{pmatrix} \in \mathbb{R}^{L\times d}$ and $\pth{\Qc}^T=\begin{pmatrix}
0 & I_{d-L} \end{pmatrix} \in \mathbb{R}^{d-L\times d}$. Observe that $\xk = \Qk x$, $\xc = \Qc x$ for $x\in \mathbb{R}^d,$ and that $I_d = \Qk\pth{\Qk}^T + \Qc\pth{\Qc}^T.$ Due to the model assumption (see Section~\ref{sec: preliminaries}), we can write
\begin{align*}
r(x_n,a) &= \inner{x_n,w^*_a} + \eta_n\\
&= \inner{\Qk\pth{\Qk}^Tx_n,w^*_a} + \inner{\Qc\pth{\Qc}^Tx_n,w^*_a}  +  \eta_n  \tag{$\Qk\pth{\Qk}^T + \Qc\pth{\Qc}^T = I$}\\
& = \inner{ \xk_n,(\Qk)^Tw^*_a} + \inner{ \xc_n,(\Qc)^T w^*_a}  +  \eta_n
\end{align*}
Plugging this relation into~\eqref{eq: omega LS} we get
\begin{align}
    b_a^{LS} =& \br*{\frac{1}{N_a}\sum_{n\in [N]} \indicator{a_n=a}\xk_n \pth{\xk_n}^T }^{-1} \br*{\frac{1}{N_a}\sum_{n\in [N]}\indicator{a_n=a} \xk_n \pth{\xk_n}^T  }\pth{\Qk}^T w^*_a \nonumber\\
    &+ \br*{\frac{1}{N_a}\sum_{n\in [N]} \indicator{a_n=a}\xk_n \pth{\xk_n}^T }^{-1} \br*{\frac{1}{N_a}\sum_{n\in [N]}\indicator{a_n=a} \xk_n \pth{\xc_n}^T   } \pth{\Qc}^T w^*_a \nonumber \\    
    &+ \br*{\frac{1}{N_a}\sum_{n\in [N]} \indicator{a_n=a}\xk_n \pth{\xk_n}^T }^{-1} \br*{\frac{1}{N_a}\sum_{n\in [N]}\indicator{a_n=a} \xk_n\eta_n }. \label{eq: supp omega LS explicit first relation}
\end{align}

We now analyze the three terms and apply the continuous mapping theorem iteratively (Theorem~\ref{theorem: CMT}) to prove the lemma. Let 
\begin{align}\label{eq: Y_N}
Y^{(i)}_N \coloneqq \frac{1}{N_a}\sum_{n\in [N]} \indicator{a_n=a}\xk_n \pth{\xk_n}^T = \frac{1}{N_a/N}\frac{1}{N}\sum_{n\in [N]} \indicator{a_n=a}\xk_n \pth{\xk_n}^T = \frac{Z_N^{(2)}}{Z_N^{(1)}},\end{align}
where $Z_N^{(1)} \coloneqq N_a/N = \frac{1}{N}\sum_{n\in[N]} \indicator{a_n=a}$ and $Z_N^{(2)} \coloneqq \frac{1}{N}\sum_{n\in [N]} \indicator{a_n=a}\xk_n \pth{\xk_n}^T.$  

By the strong law of large numbers (SLLN) 
\begin{align*}
    &Z_N^{(1)} \overset{a.s}{\rightarrow}  P^{\pi_b}(a)\\
    &Z_N^{(2)} \overset{a.s}{\rightarrow}  P^{\pi_b}(a)\Eb{\xk \pth{\xk}^T\mid a},
\end{align*}
for $N\rightarrow \infty$ since both are empirical averages of $N$ i.i.d. random variables and, thus, converge to their means a.s.. 

The mean of a random variable in the empirical average $Z_N^{(1)}$ is simply given by
\begin{align*}
    \Eb{\indicator{a_n=a}}=  P^{\pi_b}(a),
\end{align*}
since all the random variables are i.i.d.. 

The mean of a random variable in the empirical average of $Z_N^{(2)}$ is given by
\begin{align*}
    \Eb{\indicator{a_n=a}\xk_n \pth{\xk_n}^T} &= P^{\pi_b}(a) \Eb{\xk_n \pth{\xk_n}^T \mid a},
\end{align*}
since $\E{X\mid A} =\frac{\E{\indicator{A}X} }{P(A)}$.

By the continuous mapping theorem (Theorem~\ref{theorem: CMT}) we get that $$Y^{(i)}_N\overset{a.s}{\rightarrow} \frac{P^{\pi_b}(a)\Eb{\xk \pth{\xk}^T\mid a}}{P^{\pi_b}(a)} =\Eb{\xk \pth{\xk}^T\mid a},$$ which is valid since $g(a,b)=\frac{a}{b}$ is continuous at $b>0$ and we assume that for all $a\in \A$ $P^{\pi_b}(a)>0.$ 

Similar reasoning, leads to the following convergence
\begin{align}
    Y^{(ii)}_N \coloneqq \frac{1}{N_a}\sum_{n\in [N]} \indicator{a_n=a}\xk_n \pth{\xc_n}^T \overset{a.s}{\rightarrow}  \Eb{\xk \pth{\xc}^T \mid a} ,\label{eq: Yb supp}
\end{align}
and
\begin{align}\label{eq: Yc supp}
    &Y^{(iii)}_N \coloneqq \frac{1}{N_a}\sum_{n\in [N]} \indicator{a_n=a}\xk_n\eta_n \overset{a.s}{\rightarrow} 0,
\end{align}
where in the last relation we also use the fact that $\eta_n$ is zero mean random variable, $ \Eb{\eta_n}=0$, and is independent of $x_n,a_n$ and thus $\eta_n$ is also independent of $\xk_n = \Qk x_n,a_n$.

By~\eqref{eq: supp omega LS explicit first relation} and by definitions~\eqref{eq: Y_N},~\eqref{eq: Yb supp},~\eqref{eq: Yc supp},
\begin{align*}
    b_a^{LS} = (Y^{(i)}_N)^{-1}Y^{(i)}_N \Qk w^*_a + (Y^{(i)}_N)^{-1}Y^{(ii)}_N \Qc w^*_a + (Y^{(i)}_N)^{-1}Y^{(iii)}_N.
\end{align*}

We now apply the continuous mapping theorem (Theorem~\ref{theorem: CMT}) on each of the three terms. Notice that the conditions of this theorem are satisfied since the limit of  $Y^{(i)}_N$, $ \Eb{\xk \pth{\xk}^T \mid a}$, has an inverse by the assumption which implies that the limit point is continuous. Thus, we get that for $N\rightarrow \infty$ it holds a.s. that
\begin{align*}
    b_a^{LS} = \pth{\Qk}^T w_a^* + R_{11}(a)^{-1} R_{12}(a)\pth{\Qc}^T w_a^*
\end{align*}
where $R_{11}(a) = \Eb{\xk\pth{\xk}^T | a}, \ R_{12}(a) =  \Eb{\xk\pth{\xc }^T | a}.$ By taking union bound on all $a\in \A$ we conclude the proof.
\end{proof}

\newpage
\section{Linear Regression with Side Information} 
\label{supp: properties of prr}

We wish to construct a least squares variant for $w^*_a$, which utilizes the information in Equation~\eqref{eq: side information}. In Section~\ref{sec: side information} we considered the problem of linear regression under the linear model ${y = \inner{x, Pw^*} + \eta}$, where $P \in \R^{d \times d}$ is a projection matrix of rank $m$ and $\eta$ is some centered random noise. One way to solve this problem is through ridge regression on the full euclidean space $\R^d$, under projection of $w$. In Section~\ref{sec: side information} we constructed a projected variant of ridge regression (P-RR), which attempts to solve the regularized problem in Equation~\ref{eq: prr optimization problem}, i.e.,
\begin{align*}
    \min_{w\in \mathbb{R}^d}\brk[c]*{\sum_{i=1}^{t-1} \pth{\inner{x_i, P_aw}-y_{a,i}}^2 + \lambda\norm{P_aw}_2^2 },
\end{align*}
We then took its smallest norm solution
\begin{align}
\label{eq: prr solution}
    \hat{w}_{t,a}^{P_a}
    = 
    \pth{P_a \pth{\lambda I_d + \sum_{i=1}^{t-1} x_i x_i^T} P_a}^\dagger
    \pth{\sum_{i=1}^{t-1} y_i x_i}.
\end{align}
where $y_{a,i} =  r_i - \inner{x_i, M_{a_i}^\dagger b_{a_i}}$. 

Let us now take a closer look at Equation~\eqref{eq: prr solution}. We wish to show that this solution is closely related to a ridge regression problem in a lower dimension. First, notice that taking the pseudoinverse of $P_a \pth{\lambda I_d + \sum_{i=1}^{t-1} x_i x_i^T} P_a$ can be written in an equivalent method using the inverse operator. This is formalized generally in the following proposition.

\begin{proposition}
\label{prop: equivalnce of psudoinverse and low-dimension counters} Let $V\in \mathbb{R}^{d\times d}$ be a PD matrix and $P=UU^T$ a projection matrix of rank $l$, where and $U\in \mathbb{R}^{d\times l}$ is a matrix with orthonormal columns. Then,
\begin{align*}
    (PVP)^\dagger = U(U^TVU)^{-1}U^T.
\end{align*}
\end{proposition}
\begin{proof}
Observe that
\begin{align}
    PVP = U(U^TVU)U^T  \label{eq: PRR rel 1}
\end{align}

Let $U_{\mathrm{Ext}}\in \mathbb{R}^{d\times d}$ be a unitary matrix with its first $l$ columns $U$ and rest $d-l$ columns be arbitrary orthonormal vectors (such that $U_{\mathrm{Ext}}$ is unitary), i.e., an extension of $U$ to the entire space $\mathbb{R}^d$. Using this notation \eqref{eq: PRR rel 1} can be written as
\begin{align}
     PVP = U_{\mathrm{Ext}}\begin{pmatrix}
        U^TVU& 0\\
         0 & 0
     \end{pmatrix} U_{\mathrm{Ext}}^T.
     \label{eq: PRR rel 2}
\end{align}
Next, recall that for any unitary matrix $\bar{U}$ and any matrix $A$ it holds that $(\bar{U}A\bar{U}^T)^\dagger=\bar{U}A^\dagger \bar{U}^T$. Then, by Equation~\eqref{eq: PRR rel 2} 
\begin{align}
    \br*{PVP}^\dagger 
    &= \br*{U_{\mathrm{Ext}} \begin{pmatrix}
      U^TVU& 0 \\
         0 & 0
     \end{pmatrix} U_{\mathrm{Ext}}^T}^\dagger \nonumber \\
    &= U_{\mathrm{Ext}} \begin{pmatrix}
        U^TVU& 0 \\
         0 & 0
     \end{pmatrix}^\dagger U_{\mathrm{Ext}}^T \tag{$(UAU^T)^\dagger=UA^\dagger U^T$}\nonumber\\
     &= U_{\mathrm{Ext}} \begin{pmatrix}
         (U^TVU)^\dagger& 0\\
         0 & 0
     \end{pmatrix}U_{\mathrm{Ext}}^T \nonumber\\
     &=U_{\mathrm{Ext}} \begin{pmatrix}
         (U^TVU)^{-1}& 0\\
         0 & 0
     \end{pmatrix}U_{\mathrm{Ext}}^T = U(U^TVU)^{-1}U^T, \label{eq: PRR rel 3}
\end{align}
where the third relation holds since $U^TVU$ is full rank (since $V$ is PD so is $U^TVU$). 
\end{proof}

\newpage
Using Proposition~\ref{prop: equivalnce of psudoinverse and low-dimension counters} we now show that P-RR is in fact equivalent to solving a ridge regression problem in a lower dimension. This will become useful in our proof of Theorem~\ref{thm: regret bound oful}, as it will allow us to reduce the linear bandit problem to a lower dimensional variant of the same problem. The following proposition proves this equivalence.

\begin{proposition}[Equivalent forms of P-RR]
\label{prop: equivalent forms of PRR}
Let $w_{\mathrm{PRR}}\in \mathbb{R}^d$ be the least $L_2$-norm solution of the following P-RR for $\lambda>0$
\begin{align*}
    \arg\min_{w\in \mathbb{R}^d}\br*{\sum_{n=1}^t \br*{\inner{Px_n, w}-y_n}^2 + \lambda\norm{Pw}_2^2 },
\end{align*}
where $x_n\in \mathbb{R}^d, y_n\in\mathbb{R}$, $P=UU^T$ is a projection matrix of rank $l$ and $U\in \mathbb{R}^{d\times l}$ is a matrix with orthonormal columns. Define $X\in \mathbb{R}^{t \times d}$ and $\tilde{X}\in \mathbb{R}^{t\times l}$ as matrices with $\brc*{x_n^T}_{n=1}^t$ and $\brc*{(U^Tx_n)^T}_{n=1}^t$ in their rows, respectively. Define $Y\in \mathbb{R}^t$ as the vector of $\brc*{y_n}_{n=1}^t$.   Then, $w_{\mathrm{PRR}}$ satisfies the following relations.
\begin{enumerate}
    \item $w_{\mathrm{PRR}} = (P (X^TX+\lambda I_d) P)^\dagger PX^TY$.
    \item $U^Tw_{\mathrm{PRR}} = (\tilde{X}^T \tilde{X}+\lambda I_l)^{-1}\tilde{X}^T Y.$
\end{enumerate}
\end{proposition}
\begin{proof}
\emph{Claim (1).}
The P-RR problem can we recast as
\begin{align*}
    \arg\min_{w\in \mathbb{R}^d}\br*{ w^T(P(X^TX+\lambda I_d)P)w + 2w^T PX^TY}.
\end{align*}
The smallest $L_2$ norm solution of this optimization problem is given by 
\begin{align}
    w_{\mathrm{PRR}} = (P (X^TX+\lambda I_d) P)^\dagger PX^TY \label{eq: PRR claim 1}
\end{align}
which establishes the first claim. 

\emph{Claim (2).} 

By Proposition~\ref{prop: equivalnce of psudoinverse and low-dimension counters} it holds that 
\begin{align}
    &(P (X^TX+\lambda I_d) P)^\dagger = U(U^T (X^TX+\lambda I_d) U)^{-1}U^T  \tag{Prop.~\ref{prop: equivalnce of psudoinverse and low-dimension counters}} \nonumber \\
    &=U(U^TX^TXU+\lambda I_l )^{-1}U^T \nonumber \tag{$U^TU= I_l$} \\
    &=U(\tilde X^T\tilde X+\lambda I_l )^{-1}U^T, \label{eq: equivalence eq 12}
\end{align}
by defining $\tilde X=XU$, which is a matrix with $\brc*{U^T x_n}$ in its rows. To conclude the proof, observe the following holds
\begin{align*}
     U^Tw_{\mathrm{PRR}} &= U^T\br*{U(\tilde{X}^T \tilde{X}+ \lambda I_l)^{-1}U^T}PX^TY \tag{By~\eqref{eq: PRR claim 1} and~\eqref{eq: equivalence eq 12}}\\
     &=U^TU(\tilde{X}^T \tilde{X}+ \lambda I_l)^{-1}U^TU\tilde{X}^TY \\
     &= (\tilde{X}^T \tilde{X}+ \lambda I_l)^{-1}\tilde{X}^TY \tag{$U^TU=I_l$}.
\end{align*}
\end{proof}

\newpage
\section{OFUL with Linear Side Information}
\label{supp: oful}

\begin{algorithm}[H]
\caption{OFUL with Linear Side Information (Equivalent Form)}
\label{algo: supp projected oful equivalent form}
\begin{algorithmic}[1]
\STATE {\bf input:} $\alpha > 0, M_a \in \R^{L\times d}, b_a \in \R^L$
\STATE {\bf init:} $V_a = \lambda I, Y_a = 0, \forall a \in \A$
\FOR{$t = 1, \hdots, T$}
    \STATE  Receive context $x_t$
    \STATE  ${\hat{w}_{t,a}^{P_a} = \pth{P_a V_a P_a}^\dagger
    \pth{Y_a - \pth{V_a - \lambda I} M_a^\dagger b_a}}$
    \STATE ${\sqrt{\beta_t(\delta)} = \sigma \sqrt{(d-L) \log{\frac{K(1+t\xBound^2/\lambda)}{\delta}}} + \lambda^{1/2}\wBound}$
  \STATE $a_t,\tilde{w}_{t,a_t}\in \arg\max_{a\in \A,  w_a \in C_{t,a}} \br*{\inner{x_t,P_{a}w_a} + \inner{x_t,M^\dagger_a b_a}}$
    \STATE Play action $a_t$ and receive $r_t$
  \STATE $V_{a_t} = V_{a_t} + x_tx_t^T, Y_{a_t} = Y_{a_t} + x_t r_t$
\ENDFOR
\end{algorithmic}
\end{algorithm}

In this section we supply regret guarantees for OFUL~\citep{abbasi2011improved} with linear side information of the form $M_a w_a=b_a$ for every $a\in \A$. In Appendix~\ref{supp: approximate oful}, building on this analysis, we study the case $M_a$ is estimated in an online manner. 

\subsection{Equivalences of Update Rule}

The optimistic estimation of the reward of each arm has the form (Algorithm~\ref{algo: projected oful})
\begin{align*}
    \inner{x_t, P_a \hat{w}_{t,a}} +\norm{x}_{(P_aV_{t-1,a}P_a)^\dagger} +\inner{x_t,M^\dagger_a b_a}.
\end{align*}

In this section we establish that this update rule is equivalent to the update rule written in Algorithm~\ref{algo: supp projected oful equivalent form}. For computational purposes, it is easier to consider the version given in Algorithm~\ref{algo: projected oful}, whereas in terms of analysis, the equivalent form given in Algorithm~\ref{algo: supp projected oful equivalent form}. The following proposition proves this equivalence by providing an analytical solution to the optimisitic optimization problem in Line 7 of Algorithm~\ref{algo: supp projected oful equivalent form}.

\begin{proposition}[Equivalent Forms of OFUL's Optimistic Estimation]
Let $P=UU^T$ be an orthogonal projection matrix, $U\in \mathbb{R}^{d\times l}$ a matrix with orthonormal columns and $V\in \mathbb{R}^{d\times d}$ a PD symmetric matrix. Fix $\hat{w}\in \mathbb{R}^d,\beta\in \mathbb{R}$ and let $$\mathcal{C}=\brc{w: \norm{U^T w - U^T \hat{w}}_{\tilde{V}} \leq \sqrt{\beta}},$$
where $\tilde{V}=U^TVU \in \mathbb{R}^{l\times l}$. Then, for any fixed $x\in \mathbb{R}^d$
\begin{align*}
     \max_{w\in \mathcal{C}}\br*{\inner{x, Pw}}= \inner{x, P \hat{w}} +\sqrt{\beta} \norm{x}_{(P V P)^\dagger}
\end{align*}
\end{proposition}
\begin{proof}
We have that
\begin{align}
    \max_{w\in \mathcal{C}}\br*{\inner{x, Pw}} 
    &=  
    \inner{x, P\hat{w}} + \max_{w\in \mathcal{C}}\br*{\inner{x, P\pth{w-\hat{w}}}} \nonumber \\
    &=
    \inner{x, P\hat{w}} + \max_{w\in \mathcal{C}}\br*{\inner{U^Tx, U^Tw-U^T\hat{w}}}. 
    \label{eq: oful update rule equivalence}
\end{align}
Next,
\begin{align*}
    \inner{U^Tx, U^Tw-U^T\hat{w}}
    &\leq 
    \norm{U^Tx}_{\tilde{V}^{-1}}\norm{U^Tw-U^T\hat{w}}_{\tilde{V}} \\
    &\leq 
    \sqrt{\beta} \norm{U^Tx}_{\tilde{V}^{-1}},
\end{align*}
where the first inequality follows by Cauchy-Schwartz and $\tilde{V}=U^TVU$ is PD, and the second inequality by the assumption $\norm{U^Tw-U^T\hat{w}}_{\tilde{V}}\leq \sqrt{\beta}$ for any $w\in \mathcal{C}$. Moreover, the inequality is attained with equality for $\tilde{V}(U^Tw - U^T\hat{w}) = \frac{\sqrt{\beta}}{\norm{U^T x}_{\tilde V^{-1}}}\tilde{V}^{-1}U^T x$ (such $w$ is indeed contained in $\mathcal{C}$). Thus, we get that
\begin{align*}
    \max_{w\in \mathcal{C}}\br*{\inner{U^Tx, U^Tw-U^T\hat{w}}} = \sqrt{\beta}\norm{U^Tx}_{\tilde{V}^{-1}}.
\end{align*}
Plugging the above into Equation~\eqref{eq: oful update rule equivalence} and we get
\begin{align*}
    \max_{w\in \mathcal{C}}\br*{\inner{x, Pw}} &= \inner{x, P\hat{w}} + \sqrt{\beta}\norm{U^Tx}_{\tilde{V}^{-1}}\\
    & = \inner{x, P\hat{w}} + \sqrt{\beta}\norm{x}_{U\tilde{V}^{-1}U^T}\\
    & = \inner{x, P\hat{w}} + \sqrt{\beta}\norm{x}_{(PVP)^\dagger}. \tag{Proposition~\ref{prop: equivalnce of psudoinverse and low-dimension counters}}
\end{align*}

\end{proof}

\subsection{Proof of Theorem~1}

We are now ready to prove Theorem~\ref{thm: regret bound oful}. We first provide a sketch of the proof. Using the linear side observation we can deduce $(I-P_a)w^*_a = M^\dagger_a b_a$. Thus, we should only estimate the part of $w^*_a$ not given by the linear side observation, $P_a w^*_a$. To this end, we use the P-RR estimator (see Section~\ref{sec: side information} and Appendix~\ref{supp: properties of prr})

\begin{align}
    \hat{w}_{t+1,a} \in \arg\min_{w\in \mathbb{R}^d} \brc*{\sum_{n=1}^{t} \indicator{a_n=a}\br*{\inner{P_a x_{n}, w} - (y_{n} - \inner{x_{n}, M^\dagger_a b_a})}^2 + \lambda \norm{P_a w}_2^2}, 
    \label{supp: prr estimator full}
\end{align}

\begin{proof}
We start by defining the good event $\mathcal{G}$ as the event $\brk[c]*{U_a^T w^*_a\in C_{t,a} \forall t \geq 0, \forall a \in \A}$. By Lemma~\ref{lemma: projected subspace optimism}, $P(\mathcal{G}) > 1-\delta$.

Denote by $I_k(a)$ the $k^\text{th}$ time action $a$ was chosen in the sequence $x_1, a_1, \hdots, x_t, a_t$, that is,
    \begin{equation*}
        I_k(a)
        =
        \min\brk[c]*{t : \sum_{j=1}^{t} \indicator{a_j = a} = k},
    \end{equation*}
and denote by $N_t(a)$ the total number of times action $a$ was chosen by time $t$. Also, denote the PD matrix
\begin{align*}
    \tilde V_{t-1,a} &= \lambda I_{d-L}+ \sum_{i=1}^{t-1} \indicator{a_i=a}(U_a^T x_{i})(U^T_a x_{i})^T.
\end{align*}
Then the following relations hold for every $t\geq 0$ conditioning on the good event.
\begin{align}
r_t= &\inner{x_t,w^*_{a^*(x_t)}}
 - \inner{x_t,w^*_{a_t}} \nonumber\\
 &\leq\inner{P_{a_t}x_t, \tilde w_{a_t}} +\inner{x_t,M^\dagger_{a_t} b_{a_t}}  - \inner{x_t,w^*_{a_t}} \tag{Corollary~\eqref{corollary: update rule optimism}} \nonumber\\
 &=\inner{P_{a_t}x_t, \tilde w_{a_t}}  - \inner{P_{a_t}x_t,w^*_{a_t}} \tag{$w_a = P_a w_a + M^\dagger_a b_a$, Property~\ref{property: linear side information}}\nonumber\\ 
 &=\inner{U_{a_t}^Tx_t,  \br*{ U_{a_t}^T\tilde w_{a_t}-U_{a_t}^Tw^*_{a_t}}} \tag{$P_a=U_a U_a^T$}\nonumber\\
 &\leq \norm{U_{a_t}^Tx_t}_{\tilde V_{t-1,a_t}^{-1}}\norm{ U_{a_t}^T\hat w_{a_t}-U_{a_t}^Tw^*_{a_t}}_{\tilde V_{t-1,a_t}} + \norm{U_{a_t}^Tx_t}_{\tilde V_{t-1,a_t}^{-1}}\norm{ U_{a_t}^T\hat w_{a_t}-U_{a_t}^T\tilde{w}_{a_t}}_{\tilde V_{t-1,a_t}} \nonumber \\
 &\leq 2\sqrt{\beta_{t-1}(\delta)}\norm{U_{a_t}^Tx_t}_{\tilde V_{t-1,a_t}^{-1}}, \tag{Lemma~\ref{lemma: projected subspace optimism}}
\end{align}

Next, notice that $r_t \leq 2$, since $\inner{x_t,w_a}\in[-1,1]$. Then, using the above we get that
\begin{align}
    r_t \leq 2\sqrt{\beta_{t-1}(\delta)}\min\br*{\norm{U_{a_t}^Tx_t}_{ \tilde V_{t-1,a_t}^{-1}},1}. \label{eq: oful main rel 2}
\end{align}

Combining the above and conditioning on the good event we get that for all $t\geq 0$
\begin{align*}
  &\Regret{t}\leq \sqrt{t\sum_{i=1}^t r_t^2 } \\
    &\leq 2\sqrt{t\beta_t(\delta)\sum_{i=1}^t\min \br*{\norm{U_{a_i}^Tx_i}^2_{\tilde V_{i-1,a_i}^{-1}},1} } \tag{\eqref{eq: oful main rel 2} and $\beta_t(\delta)\geq \beta_{t'}(\delta)$ for $t\geq t'$ }\\
    &= 2\sqrt{t\beta_t(\delta)\sum_{a\in \A}\sum_{i=1 }^{N_t(a)} \min \br*{\norm{U_{a}^Tx_{I_{i}(a)}}^2_{\tilde V_{I_i(a),a}^{-1}},1} }\\
    &\leq 2\sqrt{t}\br*{\lambda^{1/2}\wBound+\sigma \sqrt{(d-L)\log{\frac{K(1+t\xBound^2/\lambda))}{\delta }}}}\sqrt{(d-L)\sum_{a\in \A} \log{\lambda + \frac{N_t(a) \xBound ^2}{d-L}}} \tag{Lemma~\ref{lemma: eliptical potential lemma abbasi}}\\
    &\leq 2\sqrt{t}\br*{\lambda^{1/2}\wBound+\sigma \sqrt{(d-L)\log{\frac{K(1+t\xBound^2/\lambda))}{\delta }}}}\sqrt{(d-L)K \log{\lambda + \frac{t\xBound^2}{(d-L)K}}}, \tag{Jensen's Inequality \& $\sum_{a\in [K]} N_t(a)=t$}
\end{align*}
which concludes the proof.

\end{proof}

\subsection{Optimism of OFUL}

The next result establishes that with high probability $U_a^T w_a^*\in C_{t,a}$, i.e., the true vector $U_a^T w_a^*$ is contained in the set $ C_{t,a}$, which is a ball around the rotated PRR estimator $\hat{w}_{t,a}$. We prove this result via reduction of the PRR estimator to lower dimensional ridge regression (see Proposition~\ref{prop: equivalent forms of PRR}).

\begin{lemma}[Projected Subspace Optimism]\label{lemma: projected subspace optimism}
Let $\hat{w}_{t+1,a}$ be the PRR estimator of $w$ with $P_a=U_aU_a^T$ and $U_a\in \mathbb{R}^{d\times d-L}$. Let $\norm{P_a w^*_a}\leq \wBound, \norm{P_a x_n}\leq \xBound $ . Then, for all $t\geq 0, a\in \A$ and for any $\delta >0$ it holds that $U_{a}^Tw^*_a\in C_{t,a}$ with probability greater than $1-\delta$, where
\begin{align*}
     C_{t,a} = \brc*{w\in \mathbb{R}^{d}: \norm{U^T_{a} \hat{w}_{t+1,a} - U^T_{a}w }_{\tilde V_{t,a}} \leq \sqrt{\beta_t(\delta)}},
\end{align*}
where $\tilde V_{t,a} = \lambda I_{d-L} + \sum_{n=1}^{t} \indicator{a_n=a} (U^T_a x_n)(U^T_a x_n)^T,$ and 
\begin{equation*}
\sqrt{\beta_t(\delta)} 
= 
\sigma \sqrt{(d-L)\log{\frac{K(1+t\xBound^2/\lambda))}{\delta }} } + \lambda^{1/2}\wBound.
\end{equation*}
\end{lemma}
\begin{proof}
Fix $a\in \A$. Since $\hat{w}_{t+1,a}$ is the smallest norm solution of the PRR, by proposition~\ref{prop: equivalent forms of PRR}, $U^T_{a} \hat{w}_{t+1,a}$ has the following form
\begin{align}
    U^T_{a} \hat{w}_{t+1,a} = (\tilde X_{t,a}^T\tilde X_{t,a} + \lambda I_{d-L})^{-1}\tilde X_{t,a}^T Y_{t,a}. \label{eq: optimism abassi our case}
\end{align}
$\tilde{X}_{t,a}$ is a matrix with $\brc*{\tilde x_{i,a}}_{i=1}^{t}$ in its rows where $\tilde{x}_{i,a} = \indicator{a_i=a}(U_a^T x_n)^T$, and $Y_{t,a}= (y_{1,a},..,y_{t,a})$  such that $$y_{i,a} = \indicator{a_i = a} \br*{\inner{ U_a^T x_i, U_a^T w^*_a} +\eta_i} =  \inner{ \indicator{a_i = a} U_a^T x_i, U_a^T w^*_a} +\indicator{a_i = a} \eta_i.$$ 

We now recast \eqref{eq: optimism abassi our case} such that we are able to apply Theorem~\ref{theorem: abassi confidence interval} in a smaller dimension $d-L$.  Since $a_i,x_i,U_a$ are $\F_{i-1}$ measurable it holds that $\tilde x_{i,a}$ is $\F_{i-1}$ measurable. Define $\tilde{\eta}_i = \indicator{a_i = a} \eta_i$. Since $\eta_i\in \F_{n}$ it holds that $\tilde \eta_i$ is $\F_{i}$ measurable. Furthermore, $\tilde \eta_i$ is $\sigma$ sub Gaussian since $\eta_i$ is $\sigma$ sub Gaussian. Define $\tilde{Y}_{t,a}= (\tilde{y}_{1,a},..,\tilde{y}_{t,a})$  where $\tilde{y}_{i,a}= \inner{\tilde x_{i},U_a^T w^*_a} + \tilde{\eta}_i$. Thus,~\eqref{eq: optimism abassi our case} can be written as
\begin{align*}
    U^T_{a} \hat{w}_{t+1,a} = (\tilde X_{t,a}^T\tilde X_{t,a} + \lambda I_{d-L})^{-1}\tilde X_{t,a}^T \tilde Y_{t,a}.
\end{align*}

We now employ Theorem~\ref{theorem: abassi confidence interval} in dimension $d-L$ since $\tilde x_i, U_a^T w^*_a\in \mathbb{R}^{d-L}$. Furthermore, $\norm{U_a^T x_i} = \norm{Px_i}\leq \xBound $ and similarly $\norm{Pw^*_a}\leq \wBound$. Taking a union bound on $a\in \A$ concludes the claim. 
\end{proof}
\begin{corollary}[Update Rule Optimism] \label{corollary: update rule optimism}
Assume $U_a w^*_a\in C_{t,a}$ for any $t\geq 0$ and for all $a\in \A$. Then,
\begin{align*}
    \inner{x_t, P_a \tilde w_{a,t}} +  \inner{x_t, M^\dagger_a b_a} \geq  \inner{x_t, w^*_a}.
\end{align*}
Thus, $\inner{x_t, w^*_{a^*(x_t)}}\leq \inner{x_t, P_{a_t} \tilde w_{a_t,t}} +  \inner{x_t, M^\dagger_{a_t} b_{a_t}}$.
\end{corollary}
\begin{proof}
For all $a\in \A$, by the update rule definition
\begin{align*}
    &\inner{x_t, P_a \tilde w_{a,t}} +  \inner{x_t, M^\dagger_a b_a} \\
    &= \max_{U_a^Tw\in C_{t-1,a}}\inner{x_t, P_a w} +  \inner{x_t, M^\dagger_a b_a}\\
    &= \max_{U_a^Tw\in C_{t-1,a}} \inner{U_a^Tx_t, U_a^T w} +  \inner{x_t, M^\dagger_a b_a}\\
    &\geq \inner{U_a^Tx_t, U_a^T w^*_a} +  \inner{x_t, M^\dagger_a b_a} \tag{Lemma~\ref{lemma: projected subspace optimism}}\\
    &=\inner{x_t, P_a w^*_a} +  \inner{x_t, M^\dagger_a b_a} =\inner{x_t,  w^*_a}. \tag{$P_a=U_aU_a^T$ \& Property~\ref{property: linear side information}}
\end{align*}

Since the latter holds for all $a\in \A$ it also holds for the maximizer, i.e.,
\begin{align*}
    &\inner{x_t, w_{a^*(x_t)}} \leq \inner{x_t, P_a \tilde w_{a^*(x_t),t}} +  \inner{x_t, M^\dagger_{a^*(x_t)} b_{a^*(x_t)}} \\
    &\leq \max_{a\in \A}\inner{x_t, P_a \tilde w_{a,t}} +  \inner{x_t, M^\dagger_a b_a}=  \inner{x_t, P_{a_t} \tilde w_{a_t,t}} +  \inner{x_t, M^\dagger_{a_t} b_{a_t}}. 
\end{align*}
\end{proof}

\newpage
\section{OFUL with Learned Subspace} 
\label{supp: approximate oful}

\begin{algorithm}[H]
\caption{OFUL with Partially Observable Offline Data (Equivalent Form)}
\label{algo: approximate projected oful equivalent form}
\begin{algorithmic}[1]
\STATE {\bf input:} $\alpha > 0, M_a \in \R^{L\times d}, b_a \in \R^L,\delta>0, T$
\STATE {\bf init:} $V_a = \lambda I, Y_a = 0, \forall a \in \A$
\FOR{$n = 1, \hdots, \lg{T}-1$}
    \STATE $\forall a\in \A$, estimate $\pert{M}_{n,a}$ and calculate $\pert{M}_{n,a},\pert{P}_{n,a}$
     \STATE $\forall a\in \A$ $V_a = \lambda I_d$ and $Y_a=0$
    \FOR{$t = 0, \hdots, 2^{n}-1$}
        \STATE  Receive context $x_t$
        \STATE  ${\hat{w}_{t,a} = \pth{\pert{P}_{n,a} V_a \pert{P}_{n,a}}^\dagger
        \pth{Y_a - \pth{V_a - \lambda I} \pert{M}_{n,a}^\dagger b_a}}$
        \STATE ${\sqrt{\beta_t(\delta)} = \br*{\sigma+\xBound\wBound(C_{B1}(\delta) + C_{B2}(\delta)\bar{t}_{n}^{-1/2})} \sqrt{(d-L) \log{\frac{2\log{T}K(1+t\xBound^2/\lambda)}{\delta}}} + \lambda^{1/2}\wBound}$
        \STATE $a_t,\tilde{w}_{t,a_t}\in \arg\max_{a\in \A,  w_a \in C_{t,a}} \br*{\inner{x_t,\pert{P}_{n,a}w_a} + \inner{x_t,\pert{M}^\dagger_{n,a} b_a}}$
        \STATE Play action $a_t$ and receive $r_t$
      \STATE $V_{a_t} = V_{a_t} + x_tx_t^T, Y_{a_t} = Y_{a_t} + x_t r_t$
    \ENDFOR
\ENDFOR
\end{algorithmic}
\end{algorithm}

Before supplying the proof we define some useful notations. 
We denote $\pert{M}_{t,a}$ as the estimated matrix  $M_a$ at time step $t$ (see~\eqref{eq: Ma estimator}), $\pert{P}_{t,a},\pert{M}^\dagger_{t,a}$ as the orthogonal projection on the kernel of $\pert{M}_{t,a}$ and the pseudo-inverse of $\pert{M}_{t,a}$, respectively. We also denote $\pert{P}_{t,a} = \pert{U}_{t,a} \pert{U}_{t,a}^T$, where $\pert{U}_{t,a}\in \mathbb{R}^{d-L\times d}$ as a matrix orthonormal vectors in its rows that span the  kernel of $\pert{M}_{t,a}$.
Let $\pert{b}_{t,a} =\pert{M}_{t,a}w_a^*$ be the result of the linear transformation of the true $w_a^*$ by the estimated $\pert{M}_{t,a}$. Although this quantity is unknown it will be very useful in our analysis. Furthermore, by Property~\ref{property: linear side information} it holds that $w_a^* = \pert{P}_{t,a}w_a^* + \pert{M}_{t,a}^\dagger \pert{b}_{t,a}$. Thus, for any $x$ it holds that
\begin{align}
    \inner{x,w_a^*} = \inner{x,\pert{P}_{t,a}} + \inner{x,\pert{M}^\dagger_{t,a}\pert{b}_{t,a}}. \label{eq: thm 3 relation 1}
\end{align}

The proof follows the same line of analysis as in the exact case, i.e., of Theorem~\ref{thm: regret bound oful}. Unlike in Theorem~\ref{thm: regret bound oful}, we do not have access to the true matrices $M^\dagger_a,P_a$, but to increasingly more accurate estimates of these matrices.

To deal with this more challenging situation we use the doubling trick. The algorithm acts in exponentially increasing episodes. In each such episode, we fix the estimation of $M$, i.e., we use the estimate of $M$ available in the beginning of the episode. 

The analysis of this algorithm amounts to study the performance of the exact algorithm (as in Theorem~\ref{thm: regret bound oful}) up to a fixed, approximated, $M_a$, which induces errors in the used $M^\dagger_a,P_a$. Finally, summing the regret on each episode, we obtain the final result.

The proof heavily relies on the convergence properties of $P_a, M^\dagger$. These are shown to converge at a rate of $O(T^{-1/2})$ in Appendix~\ref{supp: convergence of M dagger and P}. These convergence rates are due to the special structure of $M_a$, as provided by Proposition~\ref{proposition: omega LS to w} and proven in Appendix~\ref{supp: proof of prop 1}.

\subsection{Notations and the Good Event of Theorem~2}

Before supplying the proof, we define some notations. We define the good event $\G$ and prove it holds with high probability. We denote $n\in \brc*{0,..,\log{T}-1}$ as the episode index, $\bar{t}_n=2^n$ as the number of rounds performed at the beginning at the start of the $n^{th}$ episode, and $$t\in \brc*{0,..,2^n-1}=\brc*{0,..,\bar{t}_n-1 } =\brc*{0,..,\bar{t}_{n+1} - \bar{t}_{n}},$$
as the number of rounds at the $n^{th}$ episode.

Let 
\begin{align}
    \Delta M_n(\delta) = \frac{C_{B1}(\delta)}{\sqrt{\bar{t}_n}} +\frac{C_{B2}(\delta)}{\bar{t}_n} =\frac{C_{B1}(\delta)}{\sqrt{2^n}} +\frac{C_{B2}(\delta)}{2^n}, \label{definition: def Delta M thm 2}
\end{align}
characterize the convergence of the estimation of $M_a$ (see Corollary~\ref{corrollary: M estimation}), and 
\begin{align*}
    C_{t,a,n}(\delta') = \brc*{w\in \mathbb{R}^{d}: \norm{\pert{U}^T_{a,n} \hat{w}_{t+1,a} - \pert{U}^T_{a,n}w }_{\tilde V_{t,a}(\pert{U}_{a,n})} \leq \sqrt{\beta_{t,n}(\delta')}},
\end{align*}
where $\sqrt{\beta_{t,n}(\delta')}$ is defined in Lemma~\ref{lemma: optimism with subspace error} and given by
\begin{align*}
\sqrt{\beta_{t,n}(\delta')} 
= 
\br*{\sigma+\xBound\wBound  C_n(\delta) } \sqrt{(d-L) \log{\frac{K(1+t\xBound^2/\lambda)}{\delta'}}} + \lambda^{1/2}\wBound
\end{align*}
where we used $\Delta M_n(\delta)\sqrt{t}\leq\Delta M_n(\delta)\sqrt{\bar{t}_n} \equiv C_n(\delta)$, since $t\leq \bar{t}_n$ and by the definition of $\Delta M_n$.

Furthermore, denote $C_{t,a,n} \equiv C_{t,a,n}\br*{\frac{\delta}{2\log{T}}},\Delta M_n \equiv \Delta M_n\br*{\frac{\delta}{2\log{T}}}$, and define the following failure events.
\begin{align*}
    &F^{CI}_n = \brc*{ \exists t\in \brc*{0,..,\bar{t}_{n+1} - \bar{t}_{n}}, a\in \A: \pert{U}^T_{n,a} w^*_a \notin  C_{t,a,n}},\\
    &F^{M}_n = \brc*{ \exists n\in[\log{T}], a\in \A: \norm{M_a -\pert{M}_{a,n}}\geq \Delta M_n },
\end{align*}
where $\pert{M}_{a,n}$ is the estimated $M$ matrix at the beginning of the $n^{th}$ episode based on $\bar{t}_n$ samples gathered so far. Recall that $\hat P_{t,a} = \hat U_{n,a} \hat U_{n,a}^T$. The first event has to do with the rotated weights lying in the confidence ellipsoid $C_{t,a,n}$. The second event ensures the error in approximation of $M_a$ is small enough, i.e., converges at a rate given by Equation~\eqref{definition: def Delta M thm 2}.

\begin{lemma}[Good Event Holds with High Probability]\label{lemma: good event approximated oful}
Let the bad event be $\cup_{n=1}^{\log{T}}F^{CI}_n \cup_{n=1}^{\log{T}}F^{M}_n$ and the good event, $\G$, its complement. Then, $\Pr\br*{\G}\geq 1-\delta$.
\end{lemma}
\begin{proof}
    We prove $\Pr\br*{\cup_{n=1}^{\log{T}}F^{CI}_n}\leq \delta$ and $\Pr\br*{\cup_{n=1}^{\log{T}}F^{M}_n}\leq \delta$. Then, applying the union bound and re-scaling $\delta\gets \delta/2$ we conclude the proof.
    
    \emph{The failure event $\cup_{n=1}^{\log{T}}F^{CI}_n$.} Fix episode $n\in [\log{T}]$. Define the in-episode filtration $\brc*{F^n_{t}}_{t=0}^{\bar{t_n}}$, where $F^n_{t} = F_{\bar{t}_n+ t}$ where $F_{\bar{t}_n+ t}$ is the natural filtration (see Section~\ref{sec: preliminaries}). Observe that $\pert{M}_{n}$ is measurable w.r.t. to all $\brc*{F^n_{t}}_{t=0}^{\bar{t_n}}$ since it is determined at the beginning of the $n^{th}$ episode. Thus, we can apply Lemma~\ref{lemma: optimism with subspace error} which holds uniformly for all $t\in \brc{0,..,\bar{t}_{n}}$, and get
    \begin{align*}
        \Pr\br*{F_n^{CI}}\leq \delta'.
    \end{align*}
    By taking a union bound on all $n\in [\log{T}]$ and setting $\delta'=\frac{\delta}{\log{T}}$ we get that ${\Pr\br*{\cup_{n=1}^{\log{T}}F^{CI}_n}\leq \delta.}$
    
    \emph{The failure event $\cup_{n=1}^{\log{T}}F^{M}_n$.}  By Corollary~\ref{corrollary: M estimation}, for any fixed $\bar{t}_n\in [\log{T}]$ the event holds with high probability. Applying the union bound and re-scaling $\delta\gets \frac{\delta}{\log{T}}$ establishes that ${\Pr\br*{\cup_{n=1}^{\log{T}}F^{M}_n}\leq \delta}$. 
\end{proof}

\begin{remark}[The failure event $\cup_{n=1}^{\log{T}}F^{M}_n$]
The probability for this failure event can be bounded using a stopping time argument, without resorting to applying a union bound. However, for brevity, and since it does not improve the final result by much (due to the union bound applied in the failure event $\cup_{n=1}^{\log{T}}F^{CI}_n$) we use simpler union bound arguments to prove this failure event does not occur with high probability. 

\end{remark}

\subsection{Proof of Theorem~2}

\begin{proof}
We start by conditioning on the good event $\mathcal{G}$. By Lemma~\ref{lemma: good event approximated oful} it occurs with probability greater than $1-\delta$.

Observe that the cumulative regret at round $T$ (assuming $\log{T}\in \mathbb{N}$ ) is also given by the sum of episodic regret,
\begin{align}
    \Regret{T} =\sum_{n=0}^{\lg{T}-1}\RegretN{n}, \label{eq: regret to episodic regret}
\end{align}
where the episodic regret is given by 
\begin{align*}
    \RegretN{n} = \sum_{t=\bar{t}_n}^{\bar{t}_{n+1}-1} r_t\quad\mathrm{where}\quad \bar{t}_n=2^n.
\end{align*}

We now bound the episodic regret for any $n\in[\log{T}]$. Let $t\in \brc*{0,..,\bar{t}_n -1}$ be a time index of the $n^{th}$ episode. 

\begin{align}
r_t= &\inner{x_t,w_{a^*(x_t)}}
 - \inner{x_t,w_{a_t}} \nonumber\\
 &\leq\inner{\pert{P}_{n,a_t}x_t, \tilde w_{a_t}} +\inner{x_t,\pert{M}^\dagger_{n,a_t} b_{a_t}} + 2\xBound\wBound\Delta M_n  - \inner{x_t,w_{a_t}} \tag{Lemma~\ref{lemma: approximate optimism}} \nonumber\\
 &=\inner{\pert{P}_{n,a_t}x_t, \tilde w_{a_t}}  - \inner{\pert{P}_{n,a_t}x_t,w_{a_t}}  +2\xBound\wBound\Delta M_n + \inner{x_t,\pert{M}^\dagger_{n,a_t} \Delta b_{n,a_t}}\tag{By \eqref{eq: thm 3 relation 1}}\nonumber.
\end{align}

Conditioning on the good event, the last term is bounded by
\begin{align*}
    \inner{x_t,\pert{M}^\dagger_{n,a_t} \Delta b_{a_t}}\leq \norm{x_t}\norm{\pert{M}^\dagger_{n,a_t}}\norm{b_{a_t}-\pert{b}_{t,a_t}} \leq \xBound\wBound\Delta M_n,
\end{align*}
since $\norm{x}\leq S$, $\norm{\pert{M}^\dagger_{n,a_t}}\leq 1$ (By Lemma~\ref{lemma: M important properties}) and $\norm{b_{a_t}-\pert{b}_{n,a_t}}\leq \wBound \Delta M_n$ (By Lemma~\ref{lemma: M estimation error propogation}, error $\Delta M_n$ in the estimation of $M$).

Plugging this bound and setting $\pert{P}_{n,a} = \pert{U}_{n,a} \pert{U}_{n,a}^T$ we get
\begin{align*}
r_t &\leq \inner{\pert{U}_{n,a_t}^Tx_t,  \br*{ \pert{U}_{n,a_t}^T\tilde w_{a_t}-\pert{U}_{n,a_t}^Tw_{a_t}}} +3\xBound\wBound\Delta M_n \nonumber\\
 &\leq \norm{\pert{U}_{n,a_t}^Tx_t}_{\tilde V_{t-1,a_t}(\pert{U}_n)^{-1}}\norm{ \pert{U}_{n,a_t}^T\hat w_{a_t}-\pert{U}_{n,a_t}^Tw_{a_t}}_{\tilde V_{t-1,a_t}(\pert{U}_n)}  \nonumber \\
 &\quad + \norm{\pert{U}_{n,a_t}^Tx_t}_{\tilde V_{t-1,a_t}(\pert{U}_n)^{-1}}\norm{ \pert{U}_{n,a_t}^T\hat w_{a_t}-\pert{U}_{n,a_t}^T\tilde{w}_{a_t}}_{\tilde V_{t-1,a_t}(\pert{U}_n)} + 3\xBound\wBound\Delta M_n  \tag{C.S. Inequality}\nonumber \\
 &\leq 2\sqrt{\beta_{t-1}(\delta)}\norm{\pert{U}_{n,a_t}^Tx_t}_{\tilde V_{t-1,a_t}(\pert{U}_n)^{-1}} +3\xBound\wBound\Delta M_n, \tag{Lemma~\ref{lemma: optimism with subspace error}}
\end{align*}
where we define the PD matrix
\begin{align*}
    &\tilde V_{t-1,a}(\pert{U}_n) = \lambda I_{d-L}+ \sum_{i=1}^{t-1} \indicator{a_i=a} (\pert{U}_{n,a}^T x_{i})(\pert{U}_{n,a}^T x_{i})^T\\
\end{align*}

Using the fact $r_t\leq 2$ since $\inner{x_t,w_a}\in[-1,1]$ for any $a\in \A$, $\min(a+b,1)\leq \min(a,1)+\min(b,1)$ and by the above we get
\begin{align}
    r_t \leq  2\sqrt{\beta_{t-1}(\delta)}\min\br*{\norm{U_{a_t}^Tx_t}_{ \tilde V_{t-1,a_t}^{-1}},1} 
    +3\xBound\wBound\Delta M_n. \label{eq: oful approximate main rel 2}
\end{align}

 Combining the above and using Cauchy-Schwartz inequality (as in the proof of Theorem~\ref{thm: regret bound oful}) we get 
 \begin{align}
     \RegretN{n}\leq &2\sqrt{\bar{t}_n  \sum_{i=0}^{\bar{t}_n -1}\beta_i(\delta) \min\br*{\norm{\pert{U}_{n,a_i}^Tx_i}_{\tilde V_{i-1,a_i}(\pert{U}_n)^{-1}}^2,1}} + 3\xBound\wBound \bar{t}_n \Delta M_n \nonumber\\
     &\leq 2\sqrt{\bar{t}_n  \beta_{\bar{t}_{n+1}}(\delta) \sum_{i=0}^{\bar{t}_n -1} \min\br*{\norm{\pert{U}_{n,a_i}^Tx_i}_{\tilde V_{i-1,a_i}(\pert{U}_n)^{-1}}^2,1}} + 3\xBound\wBound \bar{t}_n \Delta M_n, \label{eq: oful approximate main rel 3}
 \end{align}
  where the last relation holds since $\beta_t(\delta)$ is increasing with $t$ (see its definition~\eqref{definition: beta approximation}).  We now bound the first term of~\eqref{eq: oful approximate main rel 3} with similar technique as in the proof of Theorem~\ref{thm: regret bound oful}. 
  
  Define $\tilde x^{(n)}_{i,a} = \indicator{a_i=a} \pert{U}_{n,a}x_i$ and $V^{(n)}_{i,a} = \lambda I_{d-L} +\sum_{j=1}^i \tilde x^{(n)}_{i,a}\br*{\tilde x^{(n)}_{i,a}}^T$. Importantly, since $\hat{U}_n$ is fixed for the entire $n^{th}$ episode it holds that
  \begin{align}
      V^{(n)}_{i,a} = V^{(n)}_{i-1,a} + \tilde x^{(n)}_{i,a}\br*{\tilde x^{(n)}_{i,a}}^T. \label{eq: importance of doubling trick}
  \end{align}

  Furthermore, denote by $I_{k}(n,a)$ the $k^\text{th}$ time action $a$ was chosen at the $n^{th}$ episode,
    \begin{equation*}
        I_k(a,n)
        =
        \min\brk[c]*{t\in \brc*{0,..,\bar{t}_n-1} : \sum_{j=0}^{t} \indicator{a_j = a} = k},
    \end{equation*}
and denote by $N_{\bar{t}_n-1}(a)$ the total number of times action $a$ was chosen by the end on the $n^{th}$ episode. By these definitions the following relations hold.
\begin{align*}
    &\sum_{i=0}^{\bar{t}_n -1} \min\br*{\norm{\pert{U}_{n,a_i}^Tx_i}_{\tilde V_{i-1,a_i}(\pert{U}_n)^{-1}}^2,1} \\
    &= \sum_{a\in \A}\sum_{i=0}^{N_{\bar{t}_n-1}(a)} \min\br*{\norm{\tilde{x}^{(n)}_{I_i(a,n),a}}_{ \br*{V^{(n)}_{i-1,a}}^{-1}}^2,1}\\
    &\leq 2(d-L)\sum_{a\in \A}\log{\lambda+ \frac{N_{\bar{t}_n-1}(a)\xBound^2}{d-L}}  \tag{Lemma~\ref{lemma: eliptical potential lemma abbasi}}\\
    &\leq 2(d-L)K\log{\lambda+ \frac{(\bar{t}_n-1)\xBound^2}{d-L}} \tag{Jensen's Ineq. \& $\sum_{a}N_{\bar{t}_n-1}(a) = \bar{t}_n-1$}.
\end{align*}

Plugging this back into~\eqref{eq: oful approximate main rel 3}, denote $\delta_T = \delta/(2\log{T})$, we bound $\RegretN{n}$ as follows.
\begin{align*}
    &\RegretN{n}\leq \eqref{eq: oful approximate main rel 3} \leq 2\sqrt{\bar{t}_n  \beta_{\bar{t}_{n+1}}(\delta_T)(d-L)K\log{\lambda+ \frac{\bar{t}_n\xBound^2}{d-L}}} + 3\xBound\wBound \bar{t}_n \Delta M_n\\
    &\leq 2\br*{\br*{\sigma+ \xBound\wBound\Delta M_n\sqrt{\bar{t}_n}}\sqrt{(d-L) \log{\frac{K(1+\bar{t}_n\xBound^2/\lambda)}{\delta_T}}} + \lambda^{1/2}\wBound}\sqrt{\bar{t}_n (d-L)K  \log{\lambda+ \frac{\bar{t}_n\xBound^2}{d-L}}} \\
    & + 3\xBound\wBound \bar{t}_n \Delta M_n\\
    &\leq 2\br*{\br*{\sigma+ \xBound\wBound C_{B1}(\delta_T)}\sqrt{(d-L) \log{\frac{K(1+\bar{t}_n\xBound^2/\lambda)}{\delta_T}}} + \lambda^{1/2}\wBound}\sqrt{\bar{t}_n (d-L)K  \log{\lambda+ \frac{\bar{t}_n\xBound^2}{d-L}}} \\
    &\quad +2\br*{ \xBound\wBound C_{B2}(\delta_T)\bar{t}_n^{-1/2}\sqrt{(d-L) \log{\frac{K(1+\bar{t}_n\xBound^2/\lambda)}{\delta_T}}} + \lambda^{1/2}\wBound}\sqrt{(d-L)K  \log{\lambda+ \frac{\bar{t}_n\xBound^2}{d-L}}} \\
    &\quad  + 3\xBound\wBound C_{B1}(\delta_T) \sqrt{\bar{t}_n} +3\xBound\wBound C_{B2}(\delta_T) \tag{$\Delta M_n=\frac{C_{B1}(\delta_T)}{\sqrt{\bar{t}_n}} +\frac{C_{B2}(\delta_T)}{\bar{t}_n}$ conditioned on $\G$ \& $\frac{t}{\bar{t}_{n}} \leq 1$}.
\end{align*}

Finally, using
\begin{align*}
    &\sum_{n=0}^{\lg{T}-1}\sqrt{\bar{t}_n} = \sum_{n=0}^{\lg{T}-1}2^{n/2}\leq \sqrt{2}2^{\lg{T}/2} =\sqrt{2T},\\
  & \sum_{n=0}^{\lg{T}-1} 1 \leq \lg{T},\quad \mathrm{and} \sum_{n=0}^{\lg{T}-1}\bar{t}_{n}^{1/2}\leq 3\sqrt{2},
\end{align*}
we bound the regret by
\begin{align*}
    &\Regret{T} \leq \sum_{n=0}^{\lg{T}-1} \RegretN{n}\\
    &\leq  3\br*{\br*{\sigma+\xBound\wBound C_{B1}(\delta_T)}\sqrt{(d-L) \log{\frac{K(1+t\xBound^2/\lambda)}{\delta_T}}} + \lambda^{1/2}\wBound}\sqrt{T (d-L)K  \log{\lambda+ \frac{T\xBound^2}{d-L}}}\\
    &\quad + 3\xBound\wBound C_{B1}(\delta_T) \sqrt{T} +3\sqrt{2}\xBound\wBound C_{B2}(\delta_T) (d-L)\sqrt{K\log{\lambda+ \frac{T\xBound^2}{d-L}} \log{\frac{K(1+t\xBound^2/\lambda)}{\delta_T}}}\\
    & \quad +  3\xBound\wBound C_{B2}(\delta_T).
\end{align*}
\end{proof}

\begin{remark}[Why use the Doubling Trick?]
Importantly, since $\hat{U}_{n,a}$ is fixed for all $a\in \A$ for the entire $n^{th}$ episode we can apply the elliptical potential lemma~\ref{lemma: eliptical potential lemma abbasi}, as~\eqref{eq: importance of doubling trick} holds. If we would change the estimate of $\pert{P}=\pert{U}\pert{U}^T$ at every round, a relation such as~\eqref{eq: importance of doubling trick} would not hold. We believe this problem might be alleviated by combining optimism w.r.t. the approximated subspace. We leave such study to future work.    
\end{remark}

\begin{lemma}[Update Rule Approximate Optimism]\label{lemma: approximate optimism}
Let $\norm{M_a-\pert{M}_{t,a}}\leq \Delta M$. Then, conditioning on the good event,
\begin{align*}
    \inner{x_t, w^*_{a^*(x_t)}}\leq \inner{ x_t, \pert{P}_{t,a_t} \tilde{w}_{a_t}} + \inner{x_t,\pert{M}^\dagger_{t,a_t} b_{a_t}} + 2\xBound\wBound\Delta M.
\end{align*}
\end{lemma}
\begin{proof}
    Conditioning on the good event, for all $a\in \A$ and $t\geq 1$ it holds that
    \begin{align}
        &\inner{x_t, \pert{P}_{t,a} \tilde w_{t,a} } = \inner{\pert{U}^T_{t,a} x_t, \pert{U}^T_{t,a} \tilde w_{t,a} } \nonumber  \\
        &= \max_{w \in C_{t-1,a}}\inner{\pert{U}^T_{t,a} x_t, \pert{U}^T_{t,a} \tilde w_{t,a} } \nonumber \\
        &\geq \inner{\pert{U}^T_{t,a} x_t, \pert{U}^T_{t,a} w^*_{a} } \tag{Conditioning on $\G$, Lemma~\ref{lemma: optimism with subspace error}}\nonumber \\
        &= \inner{x_t, \pert{P}^T_{t,a} w^*_{a} }.\label{eq: approximate optimism errors oful}
    \end{align}
   
    Applying this, we get the following relations hold for all $a\in \A$.
    \begin{align}
        &\inner{x_t, \pert{P}_{t,a} \tilde w_{t,a} } + \inner{x_t, \pert{M}^{\dagger}_{t,a}b_a } \geq  \inner{x_t, \pert{P}^T_{t,a} w^*_{a} } +\inner{x_t, \pert{M}^{\dagger}_{t,a}b_a } \tag{By \eqref{eq: approximate optimism errors oful}} \nonumber \\
        &=\inner{x_t, \pert{P}^T_{t,a} w^*_{a} } +\inner{x_t, \pert{M}^{\dagger}_{t,a}\pert{b}_{t,a} }  + \inner{x_t, \pert{M}^{\dagger}_{t,a}(b_a-\pert{b}_{t,a}) }\nonumber\\
        &= \inner{x_t,  w^*_{a} } +\inner{x_t, \pert{M}^{\dagger}_{t,a}(b_a-\pert{b}_{t,a}) } \tag{By \eqref{eq: thm 3 relation 1}}\nonumber\\
        &\geq \inner{x_t,  w^*_{a} } - \norm{x_t}\norm{\pert{M}^{\dagger}_{t,a}} \norm{b_a-\pert{b}_{t,a}} \tag{C.S. Inequality}\nonumber\\
        &\geq \inner{x_t,  w^*_{a} } - \xBound\wBound\Delta M,\label{eq: approximate optimism errors oful 2}
    \end{align}
where the last relation holds by bounding $\norm{x_t}\leq \xBound$, $\norm{\pert{M}_{t,a}^\dagger}\leq 1$ (Lemma~\ref{lemma: M important properties}) and the bound on $\norm{b_a-\pert{b}_{t,a}}$ follows from Lemma~\ref{lemma: M estimation error propogation}, third claim.

Thus,
\begin{align*}
    &\inner{x_t, w^*_{a^*(x_t)}} = \max_{a} \inner{x_t, w^*_{a}}\\
    &\leq \max_{a}\br*{ \inner{x_t, \pert{P}_{t,a} \tilde w_{t,a} } + \inner{x_t, \pert{M}^{\dagger}_{t,a}b_a }} +2\xBound\wBound\Delta M \\
    &= \inner{x_t, \pert{P}_{t,a_t} \tilde w_{t,a_t} } + \inner{x_t, \pert{M}^{\dagger}_{t,a_t}b_{a_t} } +2\xBound\wBound\Delta M.
\end{align*}

\end{proof}

\subsection{Optimism of OFUL with Learned Subspace}

\begin{lemma}[Projected Subspace Optimism with Subspace Error]\label{lemma: optimism with subspace error}
Assume $\pert{M}_a$ is measurable w.r.t. the filtration $\brc*{F_{t}}_{t=0}^\infty$. Assume for all $a\in \A$ $\norm{\pert{M}_a-M_a}\leq \Delta M$, let $\pert{P}_a$ be the orthogonal projection on the kernel of $\pert{M}_a$ and $\pert{M}^\dagger_a$ its psuedo-inverse. Let $\hat{w}_{t+1,a}$ be the PRR estimator w.r.t. the projection matrix $\pert{P}_a$ (see Eq.~\eqref{supp: prr estimator full}). Let $\pert{P}_a=\pert{U}_{a}\pert{U}_{a}^T$ where $\pert{U}_{a}\in \mathbb{R}^{d\times d-L}$. Assume $\norm{\pert{P}_{a}w^*_a}\leq \wBound , \norm{\pert{P}_{a} x_n}\leq S$ . Then, for all $t>0, a\in \A$ and for any $\delta\in (0,1)$ it holds that $\pert{U}_{a}^Tw^*_a\in C_{t,a}$ with probability greater than $1-\delta$, where
\begin{align*}
     C_{t,a} = \brc*{w\in \mathbb{R}^{d}: \norm{\pert{U}^T_{a} \hat{w}_{t+1,a} - \pert{U}^T_{a}w }_{\tilde V_{t,a}(\pert{U})} \leq \sqrt{\beta_t(\delta)}},
\end{align*}
where $\tilde V_{t,a}(\pert{U}) = \lambda I_{d-L} + \sum_{i=1}^{t} \indicator{a_i=a} (\pert{U}^T_{a} x_i)(\pert{U}^T_{a} x_i)^T,$ and 
\begin{equation}
\sqrt{\beta_t(\delta)} 
= 
\br*{\sigma+ \xBound\wBound\Delta M\sqrt{t}} \sqrt{(d-L) \log{\frac{K(1+t\xBound^2/\lambda)}{\delta}}} + \lambda^{1/2}\wBound. \label{definition: beta approximation}
\end{equation}
\end{lemma}
\begin{proof}
    Fix $a\in [K]$. Remember that $y_{i,a}$ is given by
    \begin{align*}
        y_{i,a} = \indicator{a_i=a}\br*{\inner{x_i,w_{a}} + \eta_i - \inner{x_i,\pert{M}^\dagger_{a} b_a}}.
    \end{align*}
    
    Plugging $\inner{x_i,w_{a}} = \inner{x_i,\pert{P}_{a} w_{a}} + \inner{x_i,\pert{M}^{\dagger}_{a} \pert{b}_{a}}$ (see Property~\ref{property: linear side information}) setting $\pert{P}_{a} =\pert{U}_{a} \pert{U}_{a}^T$ we get 
    \begin{align}
        y_{i,a} = \indicator{a_i=a}\br*{\inner{\pert{U}_{a}^T x_i,\pert{U}_{a}^Tw_{a}} + \eta_i - \inner{x_i,\pert{M}^\dagger_{a} \Delta b_{a}}},
    \end{align}
    where $\Delta b_{a} = b_a - \pert{b}_{a} $.

    Let $\tilde{X}_{t,a}\in \mathbb{R}^{d-L\times t}$ be the matrix with $\brc*{\indicator{a_i=a}\pert{U}_a x_i}_{i=1}^t$ in its rows, and let $X_{t,a}\in \mathbb{R}^{d\times t}$ be the matrix with $\brc*{\indicator{a_i=a}x_i}_{i=1}^t$ in its rows. The PRR estimator is thus given by
    \begin{align*}
        &\pert{U}^T_{a} w_{t+1,a}\\
        &=(\tilde{V}_{t,a})^{-1} \tilde{X}^T_{t,a}\tilde{X}_{t,a}  \pert{U}^T_{t,a} w_a + (\tilde{V}_{t,a})^{-1} \tilde{X}^T_{t,a} \tilde{\eta}_t + (\tilde{V}_{t,a})^{-1} \tilde{X}^T_{t,a} X_{t,a} \pert{M}^\dagger_{t,a}\Delta b_{t,a}\\
        & = \pert{U}^T_{t,a} w_a - \lambda (\tilde{V}_{t,a})^{-1}\pert{U}^T_{t,a} w_a + (\tilde{V}_{t,a})^{-1} \tilde{X}^T_{t,a} \tilde{\eta}_t + (\tilde{V}_{t,a})^{-1} \tilde{X}^T_{t,a} X_{t,a} \pert{M}^\dagger_{t,a}\Delta b_{t,a}.
    \end{align*}
    Rearranging and multiplying by $x$ both sides we get
    \begin{align*}
        &x^T\br*{\pert{U}^T_{a} w_{t+1,a} - \pert{U}^T_{a} w_a} \\
        &= - \lambda x^T(\tilde{V}_{t,a})^{-1}\pert{U}^T_{a} w_a + x^T(\tilde{V}_{t,a})^{-1} \tilde{X}^T_{t,a} \tilde{\eta}_t + x^T(\tilde{V}_{t,a})^{-1} \tilde{X}^T_{t,a} X_{t,a} \pert{M}^\dagger_{a}\Delta b_{a}\\
        &\leq \norm{x}_{(\tilde{V}_{t,a})^{-1}}\br*{ \lambda \norm{\pert{U}^T_{a} w_a}_{(\tilde{V}_{t,a})^{-1}} + \norm{\tilde{X}^T_{t,a} \tilde{\eta}_t}_{(\tilde{V}_{t,a})^{-1}} + \norm{\tilde{X}^T_{t,a} X_{t,a} \pert{M}^\dagger_{a}\Delta b_{a}}_{(\tilde{V}_{t,a})^{-1}} }
    \end{align*}
    
    Setting $x= \tilde{V}_{t,a}\br*{\pert{U}^T_{a} w_{t+1,a} - \pert{U}^T_{a} w_a}$, which implies that $$      \norm{x}_{\tilde{V}_{t,a}^{-1}} = \norm{\br*{\pert{U}^T_{a} w_{t+1,a} - \pert{U}^T_{a} w_a}}_{\tilde{V}_{t,a}},$$
    and dividing both sides by $\norm{\br*{\pert{U}^T_{a} w_{t+1,a} - \pert{U}^T_{a} w_a}}_{\tilde{V}_{t,a}}$ we get
    \begin{align}
        \norm{\br*{\pert{U}^T_{a} w_{t+1,a} - \pert{U}^T_{t,a} w_a}}_{\tilde{V}_{t,a}}
        &\leq \lambda \norm{\pert{U}^T_{a} w_a}_{(\tilde{V}_{t,a})^{-1}} + \norm{\tilde{X}^T_{t,a} \tilde{\eta}_t}_{(\tilde{V}_{t,a})^{-1}} + \norm{\tilde{X}^T_{t,a} X_{t,a} \pert{M}^\dagger_{a}\Delta b_{t,a}}_{(\tilde{V}_{t,a})^{-1}}. \label{eq: approximate projected optimism central to bound}
    \end{align}
    
    The first term of~\eqref{eq: approximate projected optimism central to bound} is bound by
    \begin{align*}
        \lambda \norm{\pert{U}^T_{a} w_a}_{(\tilde{V}_{t,a})^{-1}} \leq \lambda^{1/2} \norm{\pert{U}^T_{a} w_a}=\lambda^{1/2} \norm{\pert{P}^T_{a} w_a}\leq \lambda^{1/2} \wBound.
    \end{align*}
    
    The second term of~\eqref{eq: approximate projected optimism central to bound} bound by applying Theorem~\ref{theorem: abassi self normalized bound} and Lemma~\ref{lemma: eliptical potential lemma abbasi}, 
    \begin{align*}
        \norm{\tilde{X}^T_{t,a} \tilde{\eta}_t}_{(\tilde{V}_{t,a})^{-1}} &\leq  \sigma \sqrt{(d-L) \log{\frac{1+t\xBound^2/\lambda}{\delta}}}
    \end{align*}
    Theorem~\ref{theorem: abassi self normalized bound} is applicable by verifying its assumptions. First, $\tilde{X}_{t,a}$ is a matrix with $\indicator{a_i=a}\pert{U}_{a}x_{i}\in \mathbb{R}^{d-L}$ in its rows (which are $F_{t-1}$ measurable by the fact $\pert{U}_{a},x_i,\indicator{a_i=a}$ are $F_{t-1}$ measurable). The vector $\tilde{\eta}_t$ is a vector with $\indicator{a_i=a}\eta_{i}$ in its entries. Since $\eta_{i}$ is $F_{t-1}$ measurable and $\eta_{i}$ is $F_{t}$ measurable, $\eta_{i}$ is $F_t$ measurable. Furthermore, it is easy to verify that $\indicator{a_i=a}\eta_{i}$ is conditionally $\sigma$-sub-Gaussian w.r.t. $F_{t-1}$. 
    
    Lastly, the third term of~\eqref{eq: approximate projected optimism central to bound} is bounded  by applying the elliptical potential lemma and the assumption $\norm{M_a-\pert{M}_a}\leq \Delta M$ which implies by Lemma~\ref{lemma: M estimation error propogation} 
    \begin{align}
        \Delta b_{a}\leq \wBound \Delta M. \label{eq: error in b due to M}
    \end{align}
    
    We have that
    \begin{align*}
        &\norm{\tilde{X}^T_{t,a} X_{t,a} \pert{M}^\dagger_{a}\Delta b_{a}}_{(\tilde{V}_{t,a})^{-1}} = \norm{(\tilde{V}_{t,a})^{-1/2}\tilde{X}^T_{t,a} X_{t,a} \pert{M}^\dagger_{a}\Delta b_{a}} \\
        &\leq \norm{(\tilde{V}_{t,a})^{-1/2}\tilde{X}^T_{t,a} X_{t,a}} \norm{\pert{M}^\dagger_{a}}\norm{\Delta b_{a}} \tag{Norm is submultiplicative}\\
        &\leq \norm{(\tilde{V}_{t,a})^{-1/2}\tilde{X}^T_{t,a} X_{t,a}} \norm{\Delta b_{a}} \tag{$\norm{M_{a}^\dagger}\leq 1$, Lemma~\ref{lemma: M important properties}}\\
        &\leq \wBound \Delta M\norm{(\tilde{V}_{t,a})^{-1/2}\tilde{X}^T_{t,a} X_{t,a}} \tag{By \eqref{eq: error in b due to M}}\\
        &=\wBound \Delta M\norm{\tilde{X}^T_{t,a} X_{t,a}}_{(\tilde{V}_{t,a})^{-1}}.
    \end{align*}
    By Lemma~\ref{lemma: useful bound on path} we have
    \begin{align*}
        \norm{\tilde{X}^T_{t,a} X_{t,a}}_{(\tilde{V}_{t,a})^{-1}} \leq \xBound \sqrt{t(d-L)\log{1+ \xBound^2t/\lambda}}.
    \end{align*}
    From which we get that the third term of~\eqref{eq: approximate projected optimism central to bound} is also bounded by ($\delta\in(0,1)$)
    \begin{align*}
        \norm{\tilde{X}^T_{t,a} X_{t,a} \pert{M}^\dagger_{t,a}\Delta b_{t,a}}_{(\tilde{V}_{t,a})^{-1}}\leq  \xBound\wBound\Delta M\sqrt{t}\sqrt{(d-L)\log{\frac{1+ \xBound^2t/\lambda}{\delta}}}.
    \end{align*}
    
    Combining the above and taking union bound on $a\in [K]$.
\end{proof}

\newpage
The following lemma is based on Lemma 13 of~\citealt{lale2019stochastic}, and relies on Lemma~\ref{lemma: path bound cesa}.
\begin{lemma}[Deterministic Bound on Cumulative Visitation]\label{lemma: useful bound on path}
\begin{align*}
    \norm{\tilde{X}^T_{t,a} X_{t,a}}_{(\tilde{V}_{t,a})^{-1}} \leq \xBound \sqrt{t} \sqrt{(d-L)\log{1+t\xBound^2/\lambda}}.
\end{align*}
\end{lemma}
\begin{proof}
    The following relations hold.
    \begin{align*}
        &\norm{\tilde{X}^T_{t,a} X_{t,a}}_{(\tilde{V}_{t,a})^{-1}} = \norm{\sum_{i=1}^t (\tilde{V}_{t,a})^{-1/2}\tilde x^{(t)}_{i,a}x_i^T } \\
        &\leq  \sum_{i=1}^{t}\norm{ (\tilde{V}_{t,a})^{-1/2} \tilde x^{(t)}_{i,a}x_i^T } \tag{Triangle Inequality}\\
        &\leq \xBound \sum_{i=1}^{t}\norm{ (\tilde{V}_{t,a})^{-1/2}\tilde x^{(t)}_{i,a}} \tag{Norm is submultiplicative, \& $\norm{x}\leq \xBound $}\\
        &\leq \xBound \sum_{i=1}^{t}\norm{ (\tilde{V}_{i,a})^{-1/2}x^{(t)}_{i,a} } \tag{$(\tilde{V}_{i,a})^{-1}\succeq (\tilde{V}_{j,a})^{-1/2}$ for $j\geq i$}\\
        &\leq \xBound \sqrt{t}\sqrt{\sum_{i=1}^{t} \norm{ \tilde x^{(t)}_{i,a} }_{(\tilde{V}_{i,a})^{-1}}^2} \tag{C.S. Inequality}\\
        &\leq \xBound \sqrt{t} \sqrt{(d-L)\log{1+t\xBound^2/\lambda}}, \tag{Lemma~\ref{lemma: path bound cesa}},
    \end{align*}
    where Lemma~\ref{lemma: path bound cesa} is applied with $d-L$ (the dimension of the vectors $\tilde x_{t,a}^{(t)}$). This concludes the proof.
\end{proof}

\newpage
\section{Convergence of $M, M^\dagger, P$} 
\label{supp: convergence of M dagger and P}

Proposition~\ref{proposition: omega LS to w} establishes that from a partially observable data one is able to obtain $b_a^{LS}$ which is related to $w^*_a$ through the following linear transformation $$ b_a^{LS} = M_a w_a^*\ \mathrm{where }\ M_a =  \begin{pmatrix}
        I_{L}, & R_{11}^{-1}(a)R_{12}(a)
    \end{pmatrix}.$$ 
Although we cannot recover $w^*_a$ from this relation we can recover $(I-P) w^*a = M^\dagger_a b_a^{LS}$, i.e., the projection of $w^*a$ on the row space spanned by $M_a$ is $M^\dagger_a b_a^{LS}$. Unfortunately, $M_a$ itself depends on statistics of $\xc$, $R_{12}(a)$, which does not exist in the offline data. For brevity, we denote $b_a = b_a^{LS}.$ 

In this section we supply finite sample guarantees on the estimation of $M_a$ based on samples. First, observe that the only unknown part of $M_a$ is $R_{12}(a)$ (since $R_{11}(a)$ can be evaluated from the offline data). Thus, estimating $M_a$ is reduced to equivalent to estimating $R_{12}(a)$, i.e., estimating a sub-matrix of the full covariance matrix. 

We assume access to $t$ samples of the form $\brc*{x_i, a_i}_{i=1}^t$ where $x_i \sim \mathcal{P}_x$ and $a_i\sim \pi_b(\cdot\mid x_i)$. Using this data, which can be gathered in an online manner, we prove finite convergence guarantees for an unbiased estimate of $R_{12}(a)$, i.e., 
\begin{align*}
    &
    \hat{R}_{12,t}(a)
    =
    \frac{1}{t-1}
    \sum_{i=1}^t
    \frac{\indicator{a_i = a}}
         {P^{\pi_b}(a)}
    \xk_i \pth{\xc_i}^T,
\end{align*} 
for $t\geq 2$. Notice that indeed $\E{\hat{R}_{12,t}(a)} = R_{12}(a) = \Eb{  \xk \pth{\xc}^T \mid a}$. 

Our estimator for $R_{12}(a)$ given $t$ samples is then given by
\begin{align*}
    \hat{R}_{12,t}(a)
    =
    \hat{\Sigma}_{12}(a) + \mu_{1|a}\HmuC^T,
\end{align*}
and, naturally, the estimator for $M_a$ given $t$ samples is then
\begin{align*}
    \pert{M}_{t,a} = \begin{pmatrix}
        I_L & R_{11}(a)^{-1} \hat{R}_{12,t}(a).
    \end{pmatrix}
\end{align*}

Our approach requires access to $M_a^\dagger$ and $P_a$ (defined as the  orthogonal projection on the kernel of $M_a$). We use the plug-in estimator to obtain them both from the empirical estimator of $M_a$. Meaning
\begin{align*}
    \pert{M}^\dagger_{t,a} = \pert{M}_{t,a}^T( \pert{M}_{t,a} \pert{M}_{t,a}^T)^{-1}\quad \mathrm{and}\quad   \pert{P}_{t,a} = I - \pert{M}^\dagger_{t,a}\pert{M}_{t,a}.
\end{align*}
To establish finite sample convergence guarantees for $\pert{M}^\dagger_{t,a}$ and $\pert{P}_{t,a}$  we need to use  important properties (see Lemma~\ref{lemma: M important properties}) of $\pert{M}_{t,a}$ and $M_a$, which holds due to their very special structure, $$\rank\br*{\pert{M}_t} = \rank\br*{M}=L,\ \mathrm{and}\ \norm{M^\dagger_a},\norm{\pert{M}_{t,a}^\dagger}\leq 1.$$ 

These properties are crucial to derive the convergence of the plug-in estimator of $P_a$ and $M^\dagger_a$ from the convergence of $M_a$~\citep{wedin1973perturbation}. 

In Corollary~\ref{corrollary: M estimation} we characterize the finite-sample convergence of the estimates of $M_a$. The following lemma shows that approximation errors of $M_a$ leads to well controlled approximation errors in the approximations of $P_a,M^\dagger_a$ and $b_a$ as a result of the special structure of $M_a$.

\begin{lemma}[Deconfounder Matrix Error Propagation] \label{lemma: M estimation error propogation}
    Denote by $\norm{M_a-\pert{M}_{t,a}}$ as the estimation error of $\pert{M}_{t,a}$ relatively to $M_a$. Then,
    \begin{enumerate}
      \item $\norm{P_a- \pert{P}_{t,a}}\leq 2\norm{M_a-\pert{M}_{t,a}}$.
        \item $\norm{M_a^\dagger - \pert{M}_{t,a}^\dagger}\leq 2\norm{M_a-\pert{M}_{t,a}}$.
        \item Assuming $\norm{w^*_a}\leq \wBound $, $\norm{  b_a - \pert{b}_{t,a}} \leq  R\norm{M_a - \pert{M}_{t,a}}$.
    \end{enumerate}

\end{lemma}
\begin{proof}

\emph{Claim (1). } The second claim is a direct application of Theorem~\ref{theorem: P under perturbation}, which requires that $\rank\br*{\pert{M}_{t,a}} = \rank\br*{M_a}$. Indeed, by the first claim of Lemma~\ref{lemma: M important properties} this condition is satisfies (for any $t$ and $a$).

\emph{Claim (2). } The third claim follows by applying Theorem~\ref{theorem: M dagger perturbabation without equal rank}, by which
\begin{align*}
    \norm{M^\dagger_a - \pert{M}_{t,a}^\dagger} \leq 2\max\brc*{\norm{M^\dagger_a}^2,\norm{\pert{M}_{t,a}^\dagger}^2}\norm{M_a-\pert{M}_{t,a}}.
\end{align*}

Since both matrices $M_a,\pert{M}_{t,a}$ are of the form $\begin{pmatrix}
    I_{L} & B
\end{pmatrix}$, for some $B$, by the second claim of Lemma~\ref{lemma: M important properties} it holds that $\norm{M^\dagger_a}\leq 1,\norm{\pert{M}_{t,a}^\dagger}\leq 1$ which implies that $\norm{M^\dagger_a - \pert{M}_{t,a}^\dagger}\leq 2\norm{M_a-\pert{M}_{t,a}}.$

\emph{Claim (3). } Denote $b_a = b_{a}^{LS}$. Observe that
\begin{align*}
 &M_a w^*_a =    b_a\\
 &\pert{M}_{t,a} w^*_a = \pert{b}_{t,a}.
\end{align*}
Decreasing the two equations and taking the $L_2$ norm we get
\begin{align*}
    \norm{  b_a - \pert{b}_{t,a}} = \norm{(M_a - \pert{M}_{t,a} ) w^*_a}\leq \norm{M_a - \pert{M}_{t,a}}\norm{w^*_a}\leq \wBound \norm{M_a - \pert{M}_{t,a}}.
\end{align*}

\end{proof}

\begin{lemma}[Properties of $M$]\label{lemma: M important properties}
    Let $L \leq d$ and let $M \in \R^{L \times d}$ be the matrix defined by
    \begin{equation*}
    M
    =
    \begin{pmatrix}
        I_{L} & B
    \end{pmatrix},
\end{equation*}
where $B \in \R^{L \times (d-L)}$. Then, the following claims hold for any $B$.
\begin{enumerate}
    \item $\rank(M)=L$.
    \item $\norm{M^\dagger}\leq 1$.
\end{enumerate}
\end{lemma}
\begin{proof}
We prove that $M$ has $L$ non-zero singular values $\brc{\sigma_i}_{i=1}^L$ such that $\sigma_i\geq 1$ for all $i\in [L]$. This follows by lower bounding the minimal eigenvalue of $MM^T$. We show it is lower bounded by~$1.$ We have that
\begin{align*}
    \lambda_{\mathrm{min}}(MM^T) &= \min_{x \in \mathbb{R}^L:\norm{x}=1} \br*{x^T MM^T x} \\
    &= \min_{x \in \mathbb{R}^L:\norm{x}=1} \br*{\norm{x}+ x^TBB^Tx}\\
    &= 1 + \min_{x \in \mathbb{R}^L:\norm{x}=1} \br*{x^TBB^Tx}\geq 1,
\end{align*}
since $x^TBB^Tx = \norm{B^Tx}^2\geq 0$ for any $x$. Thus, $\lambda_{\mathrm{min}}(MM^T)\geq 1$ which implies that ${MM^T\in \mathbb{R}^{L\times L}}$ has $L$ eigenvalues $\brc{\lambda_i}_{i=1}^L$ greater than $1$. The latter implies that $M$ has exactly $L$ non-zero  singular-values, $\brc{\sigma_i}_{i=1}^L$, greater than $1$, since $\sigma_i= \sqrt{\lambda_i}\geq 1$. 

\emph{Claim (1).} Since $M$ has $L$ non-zero singular values, the rank of $M$ is $L$, since the rank of $M$ is also the total number of non-zero singular values.

\emph{Claim (2).} Let $M=U\Sigma V^T$ be the SVD decomposition of $M$. Observe that the pseudo-inverse of $M$ is also given by $M^\dagger = U\Sigma^+ V^T$ where  $(\Sigma^+)_{ii}=\frac{1}{\sigma_i}$ for non-zero $\sigma_i$ and zero otherwise. By the first claim $\sigma_i\geq 1$ for all non-zero $\sigma_i$, which implies that $\norm{M^\dagger}=\max_{i: \sigma_i \neq 0}\frac{1}{\sigma_i}\leq 1$.
\end{proof}

\begin{lemma}
[Masked Cross Correlation Estimation]
\label{thm: masked CC estimation}
Let $x, y$ be random vectors in $R^{d_1}, \R^{d_2}$, respectively, with $d_1, d_2 \geq 2$. Assume that for some $S_1, S_2 \geq 1$
\begin{align*}
    \norm{x}_2 \leq S_1 ~\text{and}~
    \norm{y}_2 \leq S_2 ~\text{almost surely}
\end{align*}
Denote
$
    R = \E{xy^T}, R_x = \E{xx^T}, R_y = \E{yy^T}.
$
For any $t \geq 1$ define
$
    \hat{R}_t 
    = 
    \frac{1}{t}\sum_{i=1}^t x_i y_i^T.
$
Then with probability at least $1-\delta$
\begin{equation*}
    \norm{\hat{R}_t - R}_2
    \leq
    S_1S_2
    \pth{
    \sqrt{
        \frac{2}{t} 
        \pth{\frac{\sqrt{\trace{R_x}\trace{R_y}}}{S_1S_2}}
        \log{\frac{d_1 + d_2}{\delta}}
    }
    + 
    \frac{4}{3t}\log{\frac{d_1 + d_2}{\delta}}
    },
\end{equation*}
\end{lemma}
\begin{proof}
    Denote $A_i = x_iy_i - R$ and notice that $\E{\frac{1}{t}\pth{A_i - R}} = 0$. Then applying Lemma~\ref{thm: matrix bernstein} we have that with probability at least $1-\delta$
    \begin{align*}
        \norm{\hat{R}_t - R}_2
        &=
        \norm{\sum_{i=1}^t A_i}_2 \\
        &\leq
        \sqrt{\frac{2 V}{t} \log{\frac{d_1 + d_2}{\delta}}}
        + 
        \frac{2}{3t}C\log{\frac{d_1 + d_2}{\delta}},
    \end{align*}
    where
    \begin{align*}
        V
        =
        \max\brc*{\norm{\E{\br*{x_iy_i^T - R}\br*{x_iy_i^T - R}^T}}_2,\norm{ \E{\br*{x_iy_i^T - R}^T\br*{x_iy_i^T - R}}}_2 }
    \end{align*}
    and $C$ is a constant chosen such that 
    $
        \norm{x_iy_i^T - R}_2 \leq C, a.s.
    $
    
    We start by bounding $V$ and next bounding $C$. We have that
    \begin{align*}
      \E{\br*{x_iy_i^T - R}\br*{x_iy_i^T - R}^T}
        &=
      \E{\pth{x_iy_i^T - R}x_iy_i^T} \\
      &=
      \E{x_ix_i^T\norm{y_i}_2^2} - RR^T \\
        &\preceq
        \E{x_ix_i^T\norm{y_i}_2^2}.
    \end{align*}
    Then, using the fact that $\E{\br*{x_iy_i^T - R}\br*{x_iy_i^T - R}^T}$ and $\E{x_ix_i^T\norm{y_i}_2^2}$ are both PSD, we have that
    \begin{align*}
        \norm{\E{\br*{x_iy_i^T - R}\br*{x_iy_i^T - R}^T}}_2
        \leq
        \norm{\E{x_ix_i^T\norm{y_i}_2^2}}_2.
    \end{align*}
    Next, by Jensen's inequality
    \begin{align*}
        \norm{\E{x_ix_i^T\norm{y_i}_2^2}}_2
        &\leq
        \E{\norm{x_ix_i^T}_2\norm{y_i}_2^2} \\
        &=
        \E{\norm{x_i}_2^2\norm{y_i}_2^2} \tag{$\norm{zz^T}_2 = \norm{z}_2^2$} \\
        &\leq
        \sqrt{\E{\norm{x_i}_2^4}\E{\norm{y_i}_2^4}} \tag{C.S.} \\
        &\leq
        S_1S_2 \sqrt{\E{\norm{x_i}_2^2}\E{\norm{y_i}_2^2}} \\
        &=
        S_1S_2 \sqrt{\trace{R_x}\trace{R_y}}
    \end{align*}
    Combining the above we have that
    \begin{align*}
        \norm{\E{\br*{x_iy_i^T - R}\br*{x_iy_i^T - R}^T}}_2
        \leq
        S_1S_2 \sqrt{\trace{R_x}\trace{R_y}}.
    \end{align*}
    Similarly,
    \begin{align*}
      \norm{\E{\br*{x_iy_i^T - R}^T\br*{x_iy_i^T - R}}}_2
        \leq
        S_1S_2 \sqrt{\trace{R_x}\trace{R_y}}.
    \end{align*}
    We therefore have that
    \begin{align*}
        V
        \leq
        S_1S_2 \sqrt{\trace{R_x}\trace{R_y}}.
    \end{align*}
    Finally we find a bound for $C$. Indeed,
    \begin{align*}
        \norm{x_iy_i^T - R}_2
        &\leq
        \norm{x_iy_i^T}_2 + \norm{R}_2 \\
        &\leq
        S_1S_2 + \norm{R}_2 \\
        &\leq
        2S_1S_2 \\
        &= C.
    \end{align*}
    This completes the proof.
\end{proof}

\begin{corollary}[Finite-Sample Analysis of $M_a$ Estimation]
\label{corrollary: M estimation}

For any $a\in \A$, let $\pert{M}_{t,a}$ be the estimation of $M_a$ based on $t$ samples (see~\eqref{eq: Ma estimator}), and $\delta>0$. Then, with probability greater than $1-\delta$,
        $$
        \norm{M_{a} - \hat{M}_{t,a}} 
        \leq \frac{C_{B1}}{\sqrt{t}} + \frac{C_{B2}(\delta)}{t},
    $$ 
where 
\begin{align*}
    &C_{B1}(\delta) =  \max_{a}\br*{\frac{\lambda_{\mathrm{min}}\br*{R_{11}(a)}^{-1}}{P^{\pi_b}(a)}}\sqrt{ 2S_1S_2 \sqrt{\trace{R_{11}}\trace{R_{22}}}  \log{\frac{d}{\delta}}},\\
    &C_{B2}(\delta) = \frac{3}{4} \max_{a}\br*{\frac{\lambda_{\mathrm{min}}\br*{R_{11}(a)}^{-1}}{P^{\pi_b}(a)}}S_1S_2\log{\frac{d}{\delta}}.
\end{align*}

\end{corollary}
\begin{proof}
  Fix $a\in \A.$ See that 
  \begin{align*}
      &M_a = \begin{pmatrix}
    I_L& R_{11}(a)^{-1} R_{12}(a)       
  \end{pmatrix},\\
    &\pert{M}_{t,a} = \begin{pmatrix}
    I_L& R_{11}(a)^{-1} \pert{R}_{t,12}(a)      \tag{By \eqref{eq: Ma estimator}} 
  \end{pmatrix}.
  \end{align*}
   
 The following relations holds.
\begin{align}
    \norm{M_a - \pert{M}_{t,a}}
    \leq 
    \norm{R_{11}(a)^{-1}}\norm{R_{12}(a) - \hat{R}_{t,12}(a)} .
    \tag{Norm is sub-multiplicative}
    \label{eq: bound autocorr error}
\end{align}
By Lemma~\ref{thm: masked CC estimation} we have that with probability at least $1-\delta$
\begin{align*}
    \norm{R_{12}(a) - \hat{R}_{t,12}(a)}
    \leq
    \frac{S_1S_2}{P^{\pi_b}(a)}
    \pth{
    \sqrt{
        \frac{2}{t} 
        \pth{\frac{\sqrt{\trace{R_{11}}\trace{R_{22}}}}{S_1S_2}}
        \log{\frac{d}{\delta}}
    }
    + 
    \frac{4}{3t}\log{\frac{d}{\delta}}
    }.
\end{align*}

Plugging back into Equation~\eqref{eq: bound autocorr error}, using $\norm{R_{11}(a)^{-1}}_2= \lambda_{\mathrm{min}}\br*{R_{11}(a)}^{-1}$, and applying the union bound on $a\in \A$ and $t\in [T]$ we conclude the first claim.
\end{proof}

\newpage
\section{Useful Results}
We restate several very useful lemmas from~\citealt{abbasi2011improved} and~\citealt{cesa2006prediction}.

\begin{theorem}[\citealt{abbasi2011improved}, Theorem 1]\label{theorem: abassi self normalized bound} Let $\brc*{F_t}_{t=1}^\infty$ be a filtration. Let $\eta_t$ be a real-values stochastic process such that $\eta_t$ is $F_t$ measurable and conditionally $\sigma$-sub-Gaussian w.r.t. $F_{t-1}$. Let $\brc*{x_t}_{t=1}^\infty$ be an $\mathbb{R}^d$-valued stochastic process such that $x_t$ is $F_{t-1}$ measurable. Assume $V$ is a $d\times d$ PD matrix. For any $t\geq 0$, define
\begin{align*}
    V_t=V+\sum_{i=1}^t x_ix_i^T\quad\quad S_t=\sum_{i=1}^t \eta_i x_i.
\end{align*}
Then, for any $\delta>0$, with probability at least $1-\delta$ for all $t\geq 0$,
\begin{align*}
    \norm{S_t}_{V_t^{-1}}^2 &\leq  2\sigma^2 \log{\frac{\det{V_t}^{1/2}\det{V}^{-1/2}}{\delta}}.
\end{align*}

\end{theorem}

\begin{theorem}[\citealt{abbasi2011improved}, Theorem 2]\label{theorem: abassi confidence interval}

Let $\brc*{\F_t}_{t=0}^\infty$ be a filtration. Let $\brc*{\eta_t}_{t=0}^\infty$ be a real-valued stochastic process such that $\eta_t$ is $\F_t$-measurable and $\eta_t$ is conditionally $\sigma$-sub-Gaussian for $\sigma\geq 0$. Let $\brc*{x_t}_{t=0}^\infty$ be an $\mathbb{R}^d$-valued stochastic process s.t. $X_t$ is $\F_{t-1}$-measurable and $\norm{x_t}\leq \xBound$. Define $y_t = \inner{x_t,w}+\eta_t$ and assume that $\norm{w}\leq \wBound $ and $\lambda>0$. Let
\begin{align*}
    \hat{w}_t = (X_t^TX_t+\lambda I_d)^{-1} X_t^T Y_t,
\end{align*}
where $X_t$ is the matrix whose rows are $x_1^T,..,x_t^T$ and $Y_t = (y_1,..,y_t)^T$. Then, for any $\delta>0$ with probability at least $1-\delta$ for all, $t\geq 0$ $w$ lies in the set
\begin{align*}
    \brc*{w\in \mathbb{R}^d: \norm{\hat{w}_t - w}_{V_t} \leq \sigma\sqrt{d\log{\frac{1+t\xBound^2/\lambda}{\delta}}} + \lambda^{1/2}\wBound}.
\end{align*}
\end{theorem}
 
\begin{lemma}[Elliptical Potential Lemma, \citealt{abbasi2011improved}, Lemma 11] \label{lemma: eliptical potential lemma abbasi}
Let $\brc*{x_t}_{t=1}^\infty$ be a sequence in $\mathbb{R}^d$ and $V_t= V+\sum_{i=1}^t x_i x_i^T$. Assume $\norm{x_t}\leq \xBound$ for all $t$. Then,
\begin{align*}
    \sum_{i=1}^t \min\br*{\norm{x_i}_{V_{i-1}^{-1}}^2,1} \leq 2\log{\frac{\det{V_t}}{\det{V}}}  \leq  2d\log{\frac{\trace{V} + t\xBound^2}{d}} -2\log{\det{V}}.
\end{align*}
\end{lemma}

\begin{lemma}[E.g,~\citealt{cesa2006prediction}, Lemma 11.11 and Theorem 11.7] \label{lemma: path bound cesa}
Let $x_1,..,x_t$ be a sequence of vectors in $\mathbb{R}^d$ and $\lambda>0$. Let $V_i = \lambda I_d +\sum_{i=j}^i x_j x_j^T$ and assume $\norm{x_i}\leq \xBound $. Then,
$$\sum_{i=1}^t \norm{x_i}_{V_i^{-1}}^2\leq d\log{1+t\xBound^2/\lambda}.$$
\end{lemma}

\begin{lemma}[Matrix Bernstein, \citealt{tropp2015introduction}, Theorem 6.1.1]
\label{thm: matrix bernstein}
Consider a finite sequence $\brk[c]*{A_k}$ of independent, random matrices with common dimension $d_1 \times d_2$. Assume that for all $k$
\begin{align*}
    \E{A_k} = 0 ~\text{and}~ \norm{A_k}_2 \leq S \quad \text{almost surely}.
\end{align*}
Denote $Z = \sum_k A_k$ and
\begin{align*}
    V(Z)
    =
    \max\brk[c]*{\norm{\E{ZZ^T}}_2, \norm{\E{ Z^TZ}}_2} =
    \max\brk[c]*{\norm{\sum_k \E{A_k A_k^T}}_2, \norm{\sum_k \E{ A_k^TA_k}}_2}.
\end{align*}
Then for all $\epsilon \geq 0$,
\begin{equation*}
    P\pth{\norm{Z}_2 \geq \epsilon}
    \leq
    \pth{d_1 + d_2}
    \exp\brk[c]*{-\frac{\epsilon^2 / 2}{V(Z) + S\epsilon/3}}.
\end{equation*}

Thus, with probability at least $1-\delta$,
\begin{align*}
    \norm{Z}_2
    \leq
    \sqrt{2 V(Z) \log{\frac{d_1 + d_2}{\delta}}} 
        +
    \frac{2}{3}S\log{\frac{d_1 + d_2}{\delta}}.
\end{align*}
\end{lemma}

\end{appendices}

\end{document}